\documentclass[12pt]{article}
\usepackage{amsmath, amsthm, amsfonts, amssymb}
\usepackage{graphicx,psfrag,epsf}
\usepackage{enumerate}
\usepackage{natbib}
\usepackage{url} % not crucial - just used below for the URL 
\usepackage{xurl}
\usepackage{booktabs}
\usepackage{array, multirow}
\usepackage{mathtools} % Added after template
\usepackage{changepage} % Added after template
\usepackage{float}
\usepackage{xcolor}

\usepackage{algorithm}
\usepackage{algpseudocode}

\usepackage[nodisplayskipstretch]{setspace}
\newcommand{\globalcolor}[1]{%
	\color{#1}\global\let\default@
}

\PassOptionsToPackage{hyphens}{url}\usepackage{hyperref}

\newcommand{\blind}{0}

\theoremstyle{plain}% Theorem-like structures provided by amsthm.sty
\newtheorem{theorem}{Theorem}

\theoremstyle{definition}
\newtheorem{definition}{Definition}

\theoremstyle{remark}

\usepackage[toc,titletoc]{appendix}

\usepackage{authblk}

\begin{document}

\def\spacingset#1{\renewcommand{\baselinestretch}%
{#1}\small\normalsize} \spacingset{1}

\renewcommand{\topfraction}{0.85}      % Max % of page top allowed for photos (default 0.7)
\renewcommand{\bottomfraction}{0.85}   % Max % of page bottom allowed for photos (default 0.3)
\renewcommand{\textfraction}{0.15}     % Min % of page that *must* be text (default 0.2)
\renewcommand{\floatpagefraction}{0.8} % Min % a float-only page must occupy (default 0.5)

% Remove the Oxford/serial comma before the "and"
\renewcommand\Authands{ and }  % Separator for the last author when n >= 3
\renewcommand\Authand{ and }   % Separator between authors when n = 2

% Customizing authblk formatting to match top-tier journals
\renewcommand\Authfont{\large\bfseries}
\renewcommand\Affilfont{\small\mdseries\itshape}
\setlength{\affilsep}{0.5em}

%%%%%%%%%%%%%%%%%%%%%%%%%%%%%%%%%%%%%%%%%%%%%%%%%%%%%%%%%%%%%%%%%%%%%%%%%%%%%%

% \if0\blind
% {
%   \title{\bf A general approach for Bayesian case influence analysis in GARCH models}
%   \author{\thanks{
%     The authors gratefully acknowledge \textit{please remember to list all relevant funding sources in the unblinded version}}\hspace{.2cm}\\
    
%     Adriano K. Suzuki \\
%     Department of ZZZ, University of WWW}
%   \maketitle
% } \fi
\if0\blind
{
	\title{\bf Generalized Neural Distributional Regression}
	
	\author[1]{Natan Hilario da Silva\thanks{Corresponding author: \href{mailto:natan.hilario@usp.br}{natan.hilario@usp.br}}}
	\author[2]{Vicente Garibay Cancho}
	\author[2]{Adriano Kamimura Suzuki}
	
	\affil[1]{Federal University of São Carlos, University of São Paulo}
	\affil[2]{Institute of Mathematics and Computer Sciences, University of São Paulo}
	
	\date{}
	\maketitle
} \fi

\if1\blind
{
	\bigskip
	\bigskip
	\bigskip
	\begin{center}
		{\LARGE\bf Generalized Neural Distributional Regression}
	\end{center}
	\medskip
} \fi

\bigskip
\begin{abstract}
	We introduce the Generalized Neural Distributional Regression (GNDR) framework, which seamlessly embeds deep neural networks into the parameter space of classical probability distributions. To reconcile the inherent non-identifiability of deep architectures with maximum likelihood theory, we propose a two-step semi-parametric estimation procedure. By isolating the terminal prediction heads and treating the upstream network as a fixed, non-linear basis expansion, GNDR enables the extraction of analytical Fisher Information matrices. This facilitates rigorous uncertainty quantification, generating observation-specific confidence bands and tolerance intervals via the multivariate Delta method. We demonstrate the framework's versatility and superior distributional calibration across diverse data modalities, including overdispersed clinical counts, right-censored transcriptomic survival profiles under a mixture cure framework, and zero-truncated age distributions derived directly from unstructured facial images. The methodology is natively implemented in the open-source Python package \textit{thetaflow}.
\end{abstract}

\noindent%
\vfill

\newpage
\spacingset{1.8} % DON'T change the spacing!

\section{Introduction}
\label{sec:intro}
The proliferation of complex data structures, ranging from genomics and high-frequency financial time series to unstructured images and natural language, poses significant challenges for modern statistical modeling. In these settings, the traditional assumption of linear predictors is overly restrictive and rarely tenable. Conversely, while contemporary machine learning algorithms excel at capturing intricate nonlinearities, they are predominantly deployed as ``black-box'' predictive engines. Consequently, they typically lack the formal uncertainty quantification, interpretability, and rigorous asymptotic properties that are fundamental to classical statistical inference.

To accommodate high-dimensional data within traditional mathematical models, researchers often rely on dimensionality reduction and pre-processing techniques. Examples include Principal Component Analysis (PCA; \citet{greenacre2022principal}), algorithmic image feature extraction \citep{kumar2014detailed}, and Autoencoders \citep{bank2023autoencoders}. While these are powerful techniques for data simplification, they are typically applied as an isolated, two-stage procedure prior to model fitting. Because this feature extraction process is uncoupled from the ultimate predictive objective, it inevitably leads to information loss; the retained features are completely agnostic to the response variable.

Modern developments in Deep Learning (DL) and Signal Processing demonstrate that this pre-processing stage can be seamlessly integrated into the model estimation procedure. For example, convolutional neural networks (CNNs) adaptively learn filter weights to capture spatial hierarchies in images, while transformer architectures encode complex sequential dependencies from text into dense vector representations. Crucially, within these frameworks, the representation layers are trained end-to-end; their parameters are optimized jointly with the predictive model via a unified loss function, ensuring that the extracted features are explicitly tailored to the primary analytical objective.

Another source of debate among statisticians arises from the fact that most machine learning algorithms function as highly parameterized, distribution-free predictive engines. While they yield highly accurate point forecasts, they traditionally lack a formal framework for uncertainty quantification (UQ) to diagnose predictive confidence. Substantial efforts have been directed toward resolving this limitation within deep learning; \citet{abdar2021review} and \citet{he2026survey} provide comprehensive surveys of state-of-the-art UQ advancements over the past decade. A prominent area of this research focuses on Bayesian Neural Networks \citep{neal2012bayesian}, which place prior distributions over network weights to establish a principled mechanism for probabilistic inference.

Despite these methodological advances, deploying such models typically requires a specialized and burdensome programmatic workflow. This process necessitates the manual specification of architectures, loss functions, and optimizers, alongside the management of low-level computational mechanics, such as mini-batch optimization and memory allocation, that are critical for convergence in high-dimensional spaces. While these engineering practices are standard among machine learning researchers, they introduce a steep learning curve for the broader statistical community. Nevertheless, integrating models capable of processing unstructured data represents a profound opportunity for statistical practice, particularly when grounded by rigorous statistical expertise.

Motivated by this opportunity, various statistical sub-disciplines are increasingly directing efforts toward incorporating deep learning architectures into classical modeling frameworks. Beyond survival analysis, where \citet{wiegrebe2024deep} provide a comprehensive survey of 61 recent approaches, similar neuro-statistical techniques are being actively adapted for time series forecasting \citep{kong2025deep}, causal inference \citep{jiao2024causal}, and spatio-temporal modeling \citep{wang2020deep}. However, current developments remain largely isolated within their respective domains. There is a distinct absence of a unifying theoretical framework capable of subsuming these disparate models as special cases. Establishing such a generalized class of models is of critical importance; it would not only provide a cohesive foundation for asymptotic inference across diverse applications but also consolidate fragmented, field-specific software into a single, robust computational ecosystem.

In this work, we formalize the integration of deep learning as a distributional regression problem. Our methodology conceptually parallels the Generalized Additive Models for Location, Scale and Shape (GAMLSS) framework \citep{rigby2005generalized}, which allows multiple parameters of a given distribution to be modeled via additive spline formulations. As we will demonstrate, replacing or augmenting traditional basis functions with neural network representations yields highly expressive models and better-calibrated residuals. While related concepts of neural distributional regression have recently emerged in the literature, independently explored by \citet{li2021deep} and \citet{rugamer2020neural} with the latter formalized in the \textit{deepregression} R package, our work introduces a distinct, generalized parameterization and training strategy.

Building upon these foundational ideas, we propose the Generalized Neural Distributional Regression (GNDR) framework. Let $Y\sim\mathcal{D}(\boldsymbol \phi)$ be a response variable governed by a $p$-dimensional parameter vector, $\boldsymbol \phi=(\phi_1,\cdots,\phi_p)^\top$, where the distribution $\mathcal{D}$ admits a twice continuously differentiable log-likelihood function. We postulate that each parameter can be mapped to an unconstrained predictor via a strictly monotonic link function $g_k$, formulated as an additive composite:
\begin{equation}
	g_k(\phi_k) = \eta_k \coloneqq \mathcal{G}_k(\boldsymbol{z}_k; \boldsymbol \xi) + \mathcal{F}_k(\boldsymbol x; \boldsymbol \omega), \quad k = 1, \cdots, p,
	\label{eq:gndr_parameters_structure}
\end{equation}
where $\mathcal{F}_k(\cdot; \cdot)$ represents the $k$-th output of an arbitrary neural network architecture $\mathcal{F}$, whose last hidden layer is fully connected (dense) to its $p$ outputs and whose output is followed by the identity activation function. This network processes an observed unstructured covariate tensor $\boldsymbol x$ and is parameterized by a tensor of trainable weights $\boldsymbol \omega$. Conversely, $\mathcal{G}_k(\cdot; \cdot)$ represents a classical, twice-differentiable structured function (such as a linear predictor) acting upon an observed covariate tensor $\boldsymbol z_k$. To guarantee structural identifiability within the $k$-th unconstrained predictor, the feature space of $\boldsymbol z_k$ processed by $\mathcal{G}_k$ must be strictly disjoint from the specific features of $\boldsymbol x$ utilized by $\mathcal{F}_k$. Notably, this constraint applies strictly element-wise for a given index $k$; the framework permits $\boldsymbol{x}$ and $\boldsymbol{z}_k$ to share covariates globally, provided those shared features do not simultaneously inform both the neural and classical components of the exact same distribution parameter $\phi_k$. Finally, $\boldsymbol \xi$ represents a vector of unknown, unconstrained, learnable parameters (e.g., classical linear coefficients) governing the structured component.

To illustrate the flexibility of this parameterization, consider the case where $\mathcal{D}$ is a normal distribution with parameter vector $\boldsymbol \phi=(\mu,\sigma^2)^\top$. We highlight four specific paradigms derived from this framework, assuming $x$ represents a scalar observed covariate.

\begin{enumerate}
	\item \textbf{Classical Normal Model (Unknown $\mu$ and $\sigma^2$):} By setting the neural components to zero ($\mathcal{F}_1 \equiv \mathcal{F}_2 \equiv 0$) and defining the structured components as global constants, we obtain:
	\begin{equation*}
		\mu = \xi_1 \quad \text{ and } \quad \log(\sigma^2) = \xi_2.
	\end{equation*}
	This recovers the standard Normal distribution without covariates, where $\boldsymbol \xi = (\xi_1, \xi_2)^\top$ acts as the vector of unconstrained parameters.
	\item \textbf{Homoscedastic Neural Regression:} By assigning a neural network to the location parameter ($\mathcal{G}_1 \equiv 0$) and a global constant to the scale parameter ($\mathcal{F}_2 \equiv 0$), we set:
	\begin{equation*}
		\mu = \mathcal{F}_1(x ; \boldsymbol \omega) \quad \text{ and } \quad \log(\sigma^2) = \xi_2.
	\end{equation*}
	This pairs a highly flexible neural regression over the mean with a static, global variance.
	\item \textbf{Heteroscedastic linear regression model:} If we apply a classical linear predictor to the mean ($\mathcal{F}_1 \equiv 0$) and a neural network to the variance ($\mathcal{G}_2 \equiv 0$), we obtain:
	\begin{equation*}
		\mu = \beta_0 + \beta_1 x \quad \text{ and } \quad \log(\sigma^2) = \mathcal{F}_2(x ; \boldsymbol \omega).
	\end{equation*}
	This model pairs a strictly interpretable linear location shift with a highly flexible variance function.
	\item \textbf{Fully Flexible Heteroscedastic Regression:} The most general non-linear framework is achieved by setting the classical components to zero ($\mathcal{G}_1 \equiv \mathcal{G}_2 \equiv 0$) and deploying a dual-output neural network:
	\begin{equation*}
		\mu = \mathcal{F}_1(x ; \boldsymbol \omega) \quad \text{ and } \quad \log(\sigma^2) = \mathcal{F}_2(x ; \boldsymbol \omega).
	\end{equation*}
	This model maximizes predictive flexibility for both parameters, albeit at the cost of classical parameter interpretability.
\end{enumerate}

These four configurations are illustrated in Figure \ref{fig:normal_distribution_examples} for a single numerical predictor. However, because $\boldsymbol x$ is processed through the network $\mathcal{F}$, the input is not restricted to tabular covariates; it can readily assume the form of unstructured images, sequential text, or high-dimensional arrays. Crucially, despite the incorporation of these complex data modalities, the framework preserves the capacity for rigorous uncertainty quantification. The asymptotic theory justifying the construction of valid prediction intervals for $Y$ and confidence intervals for the distributional parameters is detailed in Section \ref{sec:asymptotic_theory}.

This structural flexibility allows researchers to tailor the balance between feature extraction and interpretable parameter inference. For instance, a biostatistician could seamlessly integrate unstructured medical imaging ($\boldsymbol x$) with strictly linear, interpretable clinical covariates ($\boldsymbol z_j$) within the same joint likelihood. Furthermore, because parameter estimation relies on first-order stochastic optimization and mini-batching rather than computationally intensive matrix inversions, the framework scales effortlessly to massive datasets. The open-source implementation is available via the \textit{thetaflow} library in Python (\url{https://github.com/nhilariosilva/thetaflow}). Specifically for the applications presented in this work, the codes can be found in (\url{https://github.com/nhilariosilva/gndr_paper}).

\begin{figure}
	\centering
	\includegraphics[width = 0.9\linewidth]{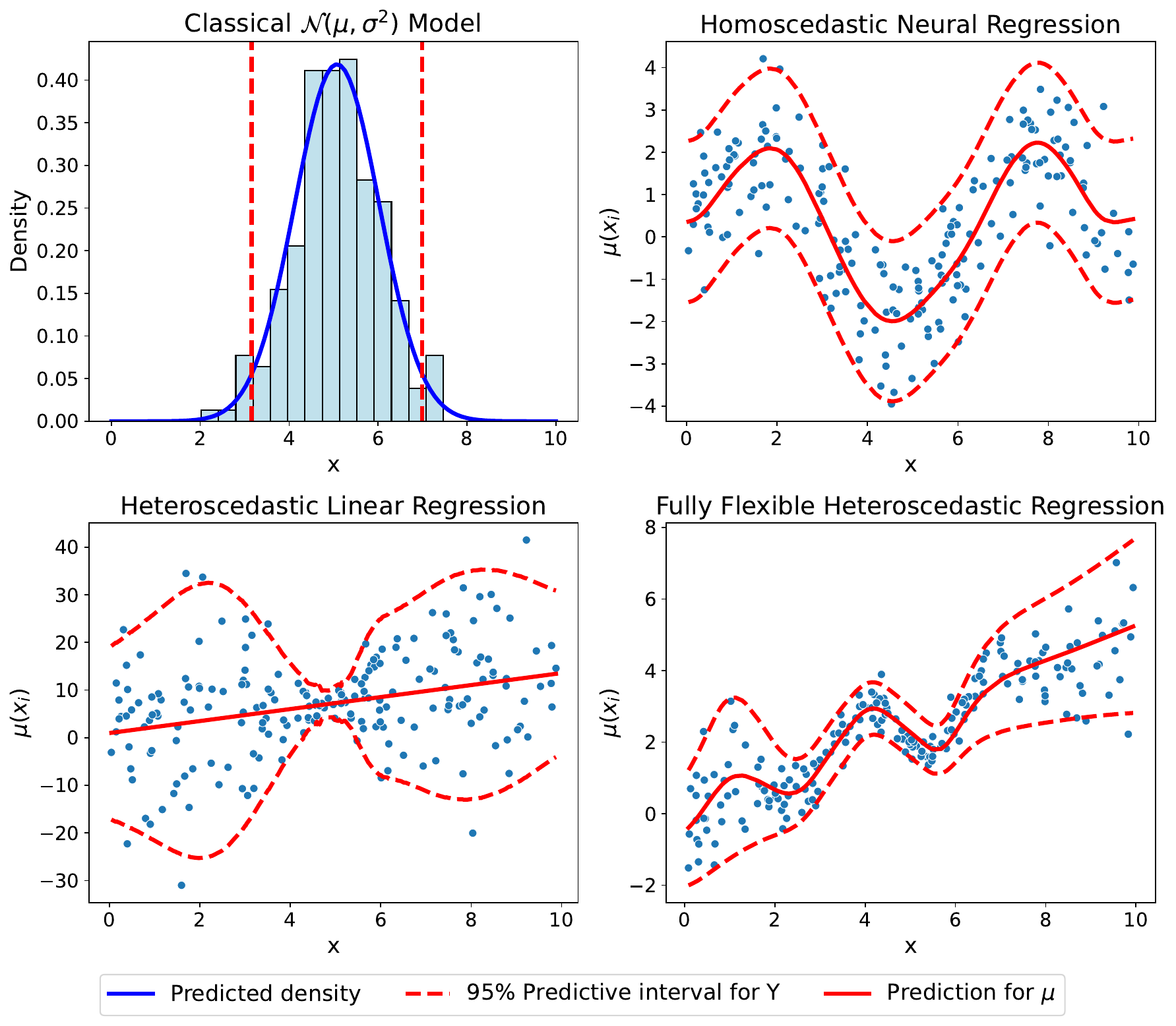}
	\caption{Four natural regression models obtained from a $\mathcal{N}(\mu, \sigma^2)$ distribution assuming neural network components.}.
	\label{fig:normal_distribution_examples}
\end{figure}

\subsection{Organization of the Paper}

The remainder of this article is organized as follows. Section 2 formalizes the proposed Generalized Neural Distributional Regression (GNDR) framework, detailing the integration of deep neural architectures into the base parametric space and discussing the critical conditions for model identifiability. In Section 3, we establish the maximum likelihood estimation procedures and derive the asymptotic properties of the proposed estimators, with the proof of the central asymptotic theorem deferred to Appendix \ref{appendix}. Furthermore, this section introduces rigorous diagnostic tools for out-of-sample model selection and goodness-of-fit evaluation. Section 4 presents a series of empirical applications that demonstrate the methodology's versatility, interpretability, and predictive performance across a diverse spectrum of data structures. Finally, Section 5 provides concluding remarks and outlines promising directions for future methodological extensions and software development.

\section{Neural Network Statistical Models}
\label{sec:methodology}
\subsection{Neural Approximators for Distribution Parameters}
\label{sec:neural_approximators}

Recalling the foundational framework defined in Equation (\ref{eq:gndr_parameters_structure}), we note that most classical statistical models based on linear predictors can be exactly recovered by assuming $\mathcal{F}_k \equiv 0$. In particular, we can mirror the semi-parametric structure of the GAMLSS family of models by restricting the structured component to a strict linear predictor, $\mathcal{G}_k(\boldsymbol{z}_k; \boldsymbol{\xi}) = \boldsymbol{z}_k^\top \boldsymbol{\beta}_k$, where $\boldsymbol{\xi}$ contains the stacked parameter vectors $\boldsymbol{\beta}_k$ for all $k \in \{1, \dots, p\}$. By leaving an arbitrary architecture for the neural component $\mathcal{F}(\boldsymbol{x}; \boldsymbol{\omega})$, the $k$-th unconstrained predictor becomes:
\begin{equation}
	g_k( \phi_k ) = \eta_k = \boldsymbol{z}_k^\top \boldsymbol{\beta}_k + \mathcal{F}_k(\boldsymbol x ; \boldsymbol \omega), \quad k = 1, \cdots, p.
	\label{eq:gndr_parameters_linear_structure}
\end{equation}
As established, to preserve structural identifiability within this additive formulation, the feature space of $\boldsymbol{z}_k$ must be strictly disjoint from the specific features of $\boldsymbol{x}$ processed by $\mathcal{F}_k$. This hybrid architecture allows structured tabular data and highly unstructured data to be jointly integrated into a generalized distributional regression model.

% Must find references for the claim: "splines, this approach scales poorly, quickly succumbing to the curse of dimensionality as the number of interacting covariates increases."
In the foundational GAMLSS framework, the non-linear component is typically formulated using additive univariate smoothing splines to capture the marginal effects of individual predictors. While multi-way interactions can be incorporated via tensor product splines, this approach scales poorly, quickly succumbing to the curse of dimensionality as the number of interacting covariates increases.

In contrast, by replacing the additive spline structure with an arbitrary deep neural network $\mathcal{F}$, the GNDR framework is able to capture high-order, non-linear feature interactions. Furthermore, since the architecture of $\mathcal{F}$ employ shared hidden layers prior to splitting into the $p$ distinct outputs, the framework naturally learns latent representations that induce dependency structures between the distributional parameters $\phi_k$ themselves, a modeling flexibility that is notoriously difficult to specify a priori within classical additive frameworks.

\subsection{Model identifiabiliy}
\label{sec:model_identifiability}

As established in Equation (\ref{eq:gndr_parameters_structure}), the native distributional parameter vector $\boldsymbol \phi$ is no longer treated as a free parameter, but rather as a deterministic function of the classical fixed effects $\boldsymbol \xi$ and the neural network weights $\boldsymbol \omega$. Consequently, $\boldsymbol \theta = (\boldsymbol \omega, \boldsymbol \xi)$ constitutes the true learnable parameters from the GNDR framework. Before formalizing the estimation of these components, we must address a fundamental geometric tension regarding the identifiability of the model, starting with the following definition.
\begin{definition}
	\label{def:identifiable_model}
	Let $Y \sim \mathcal{D}(\boldsymbol{\phi})$ be a random variable governed by a probability density (or mass) function $f(y; \boldsymbol{\phi})$. The distributional family $\mathcal{D}$ is structurally identifiable if, for any two parameter vectors $\boldsymbol{\phi}_1$ and $\boldsymbol{\phi}_2$ in the parameter space, $f(y; \boldsymbol{\phi}_1) = f(y; \boldsymbol{\phi}_2)$ for almost all $y$ implies $\boldsymbol{\phi}_1 = \boldsymbol{\phi}_2$.
\end{definition}
The condition in Definition \ref{def:identifiable_model} is necessary for classical statistical inference problem, and we assume it holds for the chosen base distribution $\mathcal{D}$. However, deep neural networks are inherently non-identifiable. Due to weight permutation symmetries (node interchangeability) and activation function scaling invariances, the mapping $\boldsymbol{\omega} \mapsto \mathcal{F}(\boldsymbol x; \boldsymbol \omega)$ is structurally non-injective \citep{yang2008structural, pourzanjani2017improving, vlavcic2022neural, bona2023parameter}. Therefore, even when $\mathcal{D}$ is identifiable with respect to $\boldsymbol \phi$, the composite mapping $\boldsymbol \theta \mapsto f(y; \boldsymbol{\phi}(\boldsymbol \theta))$ remains globally non-identifiable.

From an inferential perspective, this parameter space non-injectivity guarantees that the full Fisher Information Matrix (FIM) with respect to the complete learnable vector $\boldsymbol \theta$ will be singular, or at best, highly ill-conditioned. A singular FIM violates the regularity conditions required for standard asymptotic normality, thereby invalidating standard maximum likelihood uncertainty quantification for the deep weights. To resolve this, the GNDR framework isolates the likelihood optimization through a strictly semi-parametric paradigm. Specifically, we partition the neural weights such that $\boldsymbol{\omega} = \{\boldsymbol{\omega}_{\text{deep}}, \boldsymbol{\omega}_{\text{last}}\}$. Alongside the classical fixed effects $\boldsymbol{\xi}$, only the weights corresponding to the terminal layer of $\mathcal{F}$, denoted as $\boldsymbol{\omega}_{\text{last}}$, are treated as formal statistical parameters subject to rigorous maximum likelihood inference. All preceding internal weights, $\boldsymbol{\omega}_{\text{deep}}$, are relegated to the status of nuisance quantities, optimized strictly for non-linear feature extraction. This second training step is closely related to the known Last-Layer Laplace Approximation (LLLA) applied in the Bayesian Neural Networks Literature \citep{daxberger2021laplace}. By treating the deep network as a deterministic, learned basis expansion, this targeted bifurcation requires a specialized two-step estimation procedure, which is formalized in the subsequent section.

\section{Parameter Estimation and Goodness-of-Fit}
\label{sec:parameter_estimation}

\subsection{Parameter estimation}

Let $\mathcal{S} = \{(y_i, \boldsymbol x_i, \boldsymbol z_i)\}_{i=1}^n$ represent a sample of $n$ independent observations. Each response $y_i \sim \mathcal{D}(\boldsymbol \phi_i)$ is governed by an observation-specific distributional $p$-dimensional parameter vector $\boldsymbol \phi_i = (\phi_{1i}, \cdots, \phi_{pi})^\top$, for $i = 1, \cdots, n$. For each observation, $\boldsymbol x_i$ denotes the covariate tensor processed by the neural network, and $\boldsymbol z_i = (\boldsymbol z_{1i}, \cdots, \boldsymbol z_{pi})$ denotes the collection of covariate tensors associated with the classical components.

Under the GNDR framework, the native parameters $\boldsymbol \phi_i$ are deterministically mapped from the learnable components $\boldsymbol \omega$ and $\boldsymbol \xi$. Assuming that distribution $\mathcal{D}$ admits a closed-form probability density or mass function $f$, the conditional log-likelihood of the sample is given by:
\begin{equation}
	\ell(\boldsymbol \omega, \boldsymbol \xi \mid \boldsymbol y, \boldsymbol X, \boldsymbol Z) = \sum_{i=1}^n \log f(y_i \mid \boldsymbol \phi_i(\boldsymbol x_i, \boldsymbol z_i; \boldsymbol \omega, \boldsymbol \xi)).
\end{equation}
Under classical regularity conditions, estimation proceeds by obtaining the maximum likelihood estimators (MLE):
\begin{equation}(\widehat{\boldsymbol \omega}, \widehat{\boldsymbol \xi}) = \underset{\boldsymbol \omega, \boldsymbol \xi}{\arg\max} \ \ell(\boldsymbol \omega, \boldsymbol \xi \mid \boldsymbol y, \boldsymbol X, \boldsymbol Z).
\end{equation}
However, given the identifiability problem pointed in Section \ref{sec:model_identifiability}, the log-likelihood surface with respect to the deep weights $\boldsymbol \omega$ is notoriously non-convex, presenting local optima and saddle points. Consequently, the existence of a unique global maximum is neither theoretically guaranteed nor computationally attainable. Furthermore, standard second-order root-finding algorithms (e.g., Newton-Raphson or Fisher scoring) are mathematically intractable and prone to divergence, given the number of dimensions. 

To circumvent this, the empirical estimation of the GNDR framework relies on first-order, mini-batch stochastic optimization. The learnable parameters are updated iteratively using adaptive gradient algorithms, such as Adam \citep{kingma2014adam} or AdamW \citep{loshchilov2017fixing}, leveraging automatic differentiation natively supported by computational backends like TensorFlow \citep{tensorflow2015-whitepaper}. To ensure robust convergence and prevent overfitting to the training sample, this optimization is coupled with standard deep learning heuristics, including dynamic learning rate decay and early stopping based on validation metrics.

Let $(\widetilde{\boldsymbol{\omega}}, \widetilde{\boldsymbol{\xi}})$ denote the final point estimates obtained upon convergence of this initial training step. These estimates are typically obtained via early stopping, where the optimization trajectory is halted based on the maximization of a hold-out validation likelihood to prevent overfitting \citep{bishop2006pattern}. However, from the perspective of classical statistical inference, this creates a fundamental theoretical discrepancy. Standard maximum likelihood asymptotic theory assumes the estimator is a strict root of the score equation derived from the training sample. Furthermore, as established in Section \ref{sec:model_identifiability}, the non-identifiability of the deep weights ensures that even if a true local maximum of the training likelihood were reached, the requisite regularity conditions for the asymptotic normality of $(\widetilde{\boldsymbol{\omega}}, \widetilde{\boldsymbol{\xi}})$ would remain severely violated.

To reconcile this highly parameterized optimization with classical inferential theory, we introduce a targeted second training step. We restrict the parameter space by permanently freezing all internal layers of the neural network architecture $\mathcal{F}$ at their estimated values, $\widetilde{\boldsymbol{\omega}}_{\text{deep}}$. Under this restriction, the deep layers cease to be statistical parameters and instead function as a deterministic, fixed basis expansion mapping the unstructured data to a dense latent space. The only neural weights permitted to update are those in the terminal output layer, $\boldsymbol{\omega}_{\text{last}}$ and the unconstrained global parameters, $\boldsymbol \xi$.

By freezing the internal architecture, the model is mathematically transformed into a strictly linear Vector Generalized Additive Model (VGAM; \citet{yee1996vector}) constructed over a learned feature space. To formalize this, let $\boldsymbol{v}(\boldsymbol x_i) \coloneqq \mathcal{F}_{\text{deep}}(\boldsymbol{x}_i; \widetilde{\boldsymbol{\omega}}_{\text{deep}})$ represent the $M$-dimensional vector of latent activations output by the final hidden layer for the $i$-th observation, where $M$ is the number of neurons in that layer. Let $\boldsymbol{W}_{\text{last}}$ represent the $M \times p_\mathcal{F}$ weight matrix of the terminal dense layer, and $\boldsymbol{b}_{\text{last}}$ be its corresponding $p_\mathcal{F}$-dimensional bias vector, such that $\boldsymbol{\omega}_{\text{last}} = (\boldsymbol{W}_{\text{last}}, \boldsymbol{b}_{\text{last}})$. Here, $p_{\mathcal{F}} \leq p$ is the effective number of distributional parameters $\phi_k$ modeled with an active neural component ($\mathcal{F}_k \not\equiv 0$). For the $i$-th observation, the neural contribution to the unconstrained predictors can be expressed via exact linear algebra:
\begin{equation}
	\mathcal{F}(\boldsymbol x_i ; \boldsymbol \omega) = \left(\boldsymbol W_\text{last}\right)^\top \boldsymbol v(\boldsymbol x_i) + \boldsymbol b_\text{last}.
	\label{eq:last_layer_algebra}
\end{equation}
Substituting (\ref{eq:last_layer_algebra}) back into our generalized additive formulation in (\ref{eq:gndr_parameters_structure}), each predictor $\eta_{ki}$ becomes a strictly classical function parameterized by the identifiable, $q$-dimensional, reduced vector of parameters
\begin{equation}
	\boldsymbol \theta_\text{last} = \left(\boldsymbol \xi^\top, \operatorname{vec}(\boldsymbol{W}_{\text{last}})^\top, \boldsymbol b_{\text{last}}^\top\right)^\top,
	\label{eq:theta_last_definition}
\end{equation}
where $\operatorname{vec}(\cdot)$ denotes the column-wise vectorization operator. Here, $q = b + M p_{\mathcal{F}} + p_{\mathcal{F}}$, where $b = \dim(\boldsymbol \xi)$, represents the total number of free parameters in $\boldsymbol \theta_\text{last}$, where $M$ is the fixed dimension of the latent space $\boldsymbol{v}(\boldsymbol x_i)$. Crucially, for the Fisher Information Matrix with respect to $\boldsymbol{\theta}_\text{last}$ to be full rank, the column space of the learned design matrix generated by the latent vectors $\boldsymbol{v}(\boldsymbol x_i)$ must not be perfectly collinear with the structured design matrix generated by $\mathcal{G}(\boldsymbol{z}_k; \boldsymbol{\xi})$. This strictly algebraic requirement reinforces the necessity of the disjoint covariate space condition established in Section \ref{sec:model_identifiability}; without it, structural identifiability vanishes within the additive composition.

With the deep feature extractor $\widetilde{\boldsymbol{\omega}}_{\text{deep}}$ permanently frozen, we initialize the second training phase, or fine-tuning step, to estimate the restricted parameter vector $\boldsymbol{\theta}_\text{last}$. Because the model has mathematically collapsed into a finite-dimensional, identifiable VGLM, the optimization paradigm shifts from machine learning heuristics back to classical statistical estimation. Specifically, we abandon mini-batching by considering a full gradient approach and remove validation-based early stopping. Instead, we optimize the restricted log-likelihood strictly over the training sample until full convergence. By driving the optimization to a strict stationary point rather than prematurely halting the trajectory, we ensure that the resulting maximum likelihood estimator, $\widehat{\boldsymbol{\theta}}_\text{last}$, is an exact root of the restricted training score function. Consequently, the empirical score vector evaluates to exactly zero, satisfying the classical regularity conditions required to derive the asymptotic distribution of the model parameters in the subsequent section.

It is necessary to acknowledge that driving the unpenalized log-likelihood to a classical maximum reintroduces the risk of overfitting. However, because inference is now restricted to a classical VGLM, the severity of this risk is governed strictly by the degrees of freedom consumed by $\boldsymbol{\theta}_\text{last}$, which depend directly on the width, M, of the final hidden layer. To prevent variance inflation in the resulting estimators and ensure robust out-of-sample generalization, it is highly recommended to design the neural architecture $\mathcal{F}$ such that the terminal hidden layer acts as an aggressive, parsimonious bottleneck. In empirical practice, restricting this latent dimension to a minimal integer (e.g., $M \le 4$) provides a sufficiently flexible basis expansion for most distributional mappings while preserving the statistical degrees of freedom required for stable parameter inference.

\subsection{Asymptotic theory}
\label{sec:asymptotic_theory}

Let $\widehat{\boldsymbol \theta}_\text{last}$ denote the flattened, $q$-dimensional vector of the maximum likelihood estimates obtained from the second training step.
%\begin{equation}
%	\widehat{\boldsymbol \theta}_\text{last} = \left(\widehat{\boldsymbol \xi}^\top, \operatorname{vec}(\widehat{\boldsymbol{W}}_{\text{last}})^\top, \widehat{\boldsymbol b}_{\text{last}}^\top\right)^\top,
%	\label{eq:theta_last_definition}
%\end{equation}
We can finally establish the formal asymptotic properties for this conditional estimator.

\begin{theorem}[Conditional Asymptotic Normality of the GNDR Estimator]\label{thm:asymptotic_normality}
	Consider the restricted, two-step estimation procedure defined in Section \ref{sec:parameter_estimation}. Under standard maximum likelihood regularity conditions for Vector Generalized Linear Models, and conditional on the fixed latent feature column-vector $\boldsymbol{v}(\boldsymbol x)$, as the sample size $n \to \infty$:
	\begin{equation}
		\sqrt{n} \left(\widehat{\boldsymbol \theta}_\text{last} - \boldsymbol \theta_\text{last}^{(0)} \right) \xlongrightarrow{d} \mathcal{N}_q\left(\mathbf{0}, \boldsymbol{\Sigma}_\text{last}\right),
		\label{eq:asymptotic_distribution_theta_last}
	\end{equation}
	where $\boldsymbol \theta_\text{last}^{(0)}$ represents the true population parameter vector under the assumed GNDR data-generating process. Furthermore, $\boldsymbol{\Sigma}_\text{last} = \mathcal{I}^{-1}(\boldsymbol \theta_\text{last}^{(0)})$ is the $q \times q$ asymptotic covariance matrix, defined as the inverse of the expected FIM evaluated at the true parameter.
\end{theorem}
The proof of Theorem \ref{thm:asymptotic_normality} is provided in Appendix \ref{appendix}. In empirical applications, the exact analytical form of the expected Fisher Information Matrix $\mathcal{I}(\boldsymbol \theta_\text{last})$ is generally intractable for arbitrary choices of the base distribution $\mathcal{D}$ and the learned latent representations $\boldsymbol{v}(\boldsymbol x)$. Therefore, the asymptotic covariance matrix $\boldsymbol{\Sigma}_\text{last}$ is approximated using the inverse of the observed FIM, derived from the sample log-likelihood Hessian evaluated at the maximum likelihood estimate:
\begin{equation}
	\widehat{\boldsymbol{\Sigma}}_\text{last} = \left( - \nabla^2 \ell(\widehat{\boldsymbol \theta}_\text{last} \mid \boldsymbol y, \boldsymbol X, \boldsymbol Z) \right)^{-1}.
	\label{eq:covariance_last_layer}
\end{equation}
Because the GNDR framework natively maps the parameter transformations through a strictly defined computational graph, this $q \times q$ Hessian matrix can be computed analytically with machine precision via the automatic differentiation engines (autograd) provided by modern computational backends such as TensorFlow.

Following the block concatenation of the parameter vector defined previously, the diagonal elements of $\widehat{\boldsymbol{\Sigma}}_\text{last}$ yield the estimated asymptotic variances for the identifiable model parameters. Specifically, the first $b$ diagonal entries correspond to the structured fixed effects $\widehat{\boldsymbol \xi}$, the subsequent $M p_\mathcal{F}$ entries correspond to the vectorized terminal weights $\operatorname{vec}(\widehat{\boldsymbol{W}}_{\text{last}})$, and the final $p_\mathcal{F}$ entries correspond to the terminal bias vector $\widehat{\boldsymbol b}_{\text{last}}$.

After obtaining $\widehat \Sigma_\text{last}$, deriving the covariance structure for the final parameters of interest becomes a direct application of linear algebra. Recalling Equation (\ref{eq:gndr_parameters_structure}), the neural network component $\mathcal{F}$ was initially defined with $p$ outputs. However, since the architecture permits $\mathcal{F}_k \equiv 0$ for arbitrary parameters $\phi_k$ (see Examples 1, 2 and 3 above), $\mathcal{F}$ can be strictly represented as a network yielding exactly $p_\mathcal{F}$ active outputs from its terminal dense layer, governed by a linear activation function.

Let $\boldsymbol \varphi(\boldsymbol x) = (\varphi_1(\boldsymbol x), \cdots, \varphi_{p_\mathcal{F}}(\boldsymbol x))^\top$ denote this active neural predictor vector, where $\varphi_k(\boldsymbol x) = \mathcal{F}_k(\boldsymbol x ; \boldsymbol \omega)$. We can aggregate these active components alongside the structured effects, $\boldsymbol \xi$, to form an observation-specific intermediate predictor vector:
\begin{equation}
	\boldsymbol \theta(\boldsymbol x) = (\boldsymbol \xi^\top, \boldsymbol \varphi(\boldsymbol x)^\top)^\top.
\end{equation}
Using the algebraic expansion from (\ref{eq:last_layer_algebra}) and the parameter vector definition in (\ref{eq:theta_last_definition}), this intermediate vector can be seen simply as a data-dependent linear transformation of the parameter space:
\begin{equation}
	\boldsymbol \theta(\boldsymbol x) = \boldsymbol P(\boldsymbol x) \boldsymbol \theta_\text{last},
	\label{eq:parameters_projection}
\end{equation}
where $\boldsymbol P(\boldsymbol x)$ is a block-diagonal projection matrix of dimension $(b + p_\mathcal{F}) \times q$, defined as:
\begin{equation*}
	\boldsymbol P(\boldsymbol x) = \begin{bmatrix}
		\boldsymbol I_b & \boldsymbol 0 & \boldsymbol 0\\
		\boldsymbol 0 & \boldsymbol I_{p_\mathcal{F}} \otimes \boldsymbol v(\boldsymbol x)^\top & \boldsymbol I_{p_\mathcal{F}}
	\end{bmatrix},
\end{equation*}
where $\otimes$ stands for the standard Kronecker product. Building upon Theorem \ref{thm:asymptotic_normality}, the asymptotic normality of $\widehat{\boldsymbol \theta}_\text{last}$ combined with the affine transformation described previously allows us to trivially establish the conditional asymptotic distribution of the intermediate predictor vector. By standard multivariate properties, we obtain:
\begin{equation}
	\sqrt{n} \left(\widehat{\boldsymbol \theta}(\boldsymbol x) - \boldsymbol \theta^{(0)}(\boldsymbol x) \right) \xlongrightarrow{d} \mathcal{N}_{b+p_{\mathcal{F}}}\left(\mathbf{0}, \Sigma_{\boldsymbol \theta}(\boldsymbol x)\right),
	\label{eq:asymptotic_distribution_theta}
\end{equation}
where $\boldsymbol \theta^{(0)}(\boldsymbol x) = \boldsymbol P(\boldsymbol x) \boldsymbol \theta_\text{last}^{(0)}$, and the observation-specific intermediate covariance is given by $\boldsymbol{\Sigma}_{\boldsymbol \theta}(\boldsymbol x) = \boldsymbol P(\boldsymbol x) \boldsymbol{\Sigma}_\text{last} \boldsymbol P(\boldsymbol x)^\top$.

With this intermediate result established, we possess the necessary tools to derive the asymptotic distribution for the unconstrained additive predictors $\boldsymbol \eta = (\eta_1, \dots, \eta_p)^\top$ from \eqref{eq:gndr_parameters_structure}. Let $\boldsymbol{z} = (\boldsymbol z_1, \cdots, \boldsymbol z_p)$ denote the collection of structured covariates for a given observation. From Equation (\ref{eq:gndr_parameters_structure}) and the definition of the active neural components $\boldsymbol \varphi(\boldsymbol x)$, we can define the mapping to the $k$-th unconstrained predictor as:
\begin{equation}
	g_k(\phi_k) = \eta_k = \mathcal{G}_k(\boldsymbol z_k; \boldsymbol \xi) + \varphi_k(\boldsymbol x) \coloneqq h_k(\boldsymbol \theta(\boldsymbol x), ; \boldsymbol z_k).
	\label{eq:new_gndr_parameters_structure}
\end{equation}
We can encapsulate these element-wise transformations into a single, $p$-dimensional vector-valued function, denoted $\boldsymbol h(\boldsymbol \theta(\boldsymbol x); \boldsymbol z) = (h_1(\boldsymbol \theta(\boldsymbol x); \boldsymbol z_1), \cdots, h_p(\boldsymbol \theta(\boldsymbol x); \boldsymbol z_p))^\top$. Because the classical components $\mathcal{G}_k$ may be arbitrary non-linear functions of the fixed effects $\boldsymbol{\xi}$, $\boldsymbol{h}$ is a generally non-linear, continuously differentiable mapping. Also, note that for parameters, $\phi_k$, whose neural network component is null ($\mathcal{F}_k \equiv 0$), $h_k$ becomes a function of only $\boldsymbol \xi$ and $\boldsymbol z$.

Let $\boldsymbol{J}_{\boldsymbol h}(\boldsymbol \theta(\boldsymbol x); \boldsymbol{z})$ be the $p \times (b + p_{\mathcal{F}})$ Jacobian matrix of partial derivatives of $\boldsymbol{h}$ with respect to $\boldsymbol{\theta}$, evaluated at $\boldsymbol \theta(\boldsymbol x)$. By applying the multivariate Delta method, we obtain the asymptotic distribution for the estimated unconstrained predictor vector $\widehat{\boldsymbol{\eta}}$:
\begin{equation}
	\sqrt{n} \left(\widehat{\boldsymbol \eta} - \boldsymbol \eta^{(0)} \right) \xlongrightarrow{d} \mathcal{N}_{p}\left(\mathbf{0}, \Sigma_{\boldsymbol \eta}\right),
	\label{eq:asymptotic_distribution_eta}
\end{equation}
where the true unconstrained vector is $\boldsymbol \eta^{(0)} = \boldsymbol h\left(\boldsymbol \theta^{(0)}(\boldsymbol x); \boldsymbol z\right)$, and the final asymptotic covariance matrix is analytically defined as:
\begin{equation}
	\begin{split}
		\boldsymbol{\Sigma}_{\boldsymbol \eta}(\boldsymbol x, \boldsymbol z) &= \boldsymbol{J}_{\boldsymbol h}(\boldsymbol \theta^{(0)}(\boldsymbol x); \boldsymbol{z}) \boldsymbol{\Sigma}_{\boldsymbol \theta}(\boldsymbol x) \boldsymbol{J}_{\boldsymbol h}(\boldsymbol \theta^{(0)}(\boldsymbol x); \boldsymbol{z})^\top.\\
		%	&\approx \boldsymbol{J}_{\boldsymbol h}(\widehat{\boldsymbol \theta}(\boldsymbol x); \boldsymbol{z}) \boldsymbol{\Sigma}_{\boldsymbol \theta}(\boldsymbol x) \boldsymbol{J}_{\boldsymbol h}(\widehat{\boldsymbol \theta}(\boldsymbol x); \boldsymbol{z})^\top.
	\end{split}
\end{equation}
In practice, the Jacobian matrix $\boldsymbol{J}_{\boldsymbol h}$ can be computed effortlessly with machine precision via the automatic differentiation engine. Deriving the explicit covariance matrix for the $p$ unconstrained predictors $\boldsymbol{\eta}$ is particularly advantageous for constructing confidence intervals that strictly respect the topological boundaries of the parameter space. Suppose we require a $(1-\alpha)$ confidence interval for a specific distributional parameter $\phi_k$. Leveraging the marginal asymptotic normality of $\widehat{\eta}_k$, the symmetric confidence bounds on the unconstrained scale are given by:
\begin{equation*}
	\widehat{\eta}_k \pm z_{1-\alpha/2} \sqrt{ \left( \boldsymbol{\Sigma}_{\boldsymbol \eta}(\boldsymbol x, \boldsymbol z) \right)_{kk} },
\end{equation*}
where $z_{1-\alpha/2}$ denotes the $(1-\alpha/2)$ quantile of the standard normal distribution. By mapping these unconstrained bounds through the strictly monotonic inverse link function $g_k^{-1}$, we recover a valid, typically asymmetric confidence interval for $\phi_k$ that naturally satisfies its domain constraints (e.g., strictly positive variances or bounded probabilities).

If the analytical objective is to report the asymptotic standard errors of the native parameters themselves, one can apply the multivariate Delta method a final time. We define the vector mapping from the unconstrained predictor space to the native parameter space as $\boldsymbol{g}^{-1}(\boldsymbol \eta) = \left(g_1^{-1}(\eta_1), \dots, g_p^{-1}(\eta_p)\right)^\top$. Because each inverse link function operates independently on its corresponding predictor, its associated $p \times p$ Jacobian matrix, denoted $\boldsymbol{J}_{\boldsymbol{g}^{-1}}$, is strictly diagonal and readily obtained. The final asymptotic distribution for the constrained parameter vector is thus:
\begin{equation}
	\sqrt{n} \left(\widehat{\boldsymbol \phi} - \boldsymbol \phi^{(0)} \right) \xlongrightarrow{d} \mathcal{N}_{p}\left(\mathbf{0}, \boldsymbol{\Sigma}_{\boldsymbol \phi}(\boldsymbol x, \boldsymbol z)\right),
	\label{eq:asymptotic_distribution_phi}
\end{equation}
with the theoretical asymptotic covariance matrix analytically defined as:
\begin{equation*}
	\begin{split}
		\boldsymbol{\Sigma}_{\boldsymbol \phi}(\boldsymbol x, \boldsymbol z) &= \boldsymbol{J}_{\boldsymbol{g}^{-1}}(\boldsymbol \eta^{(0)}) \boldsymbol{\Sigma}_{\boldsymbol \eta}(\boldsymbol x, \boldsymbol z) \boldsymbol{J}_{\boldsymbol{g}^{-1}}(\boldsymbol \eta^{(0)})^\top.\\
		%	&\approx \boldsymbol{J}_{\boldsymbol{g}^{-1}}(\widehat{\boldsymbol \eta}) \boldsymbol{\Sigma}_{\boldsymbol \eta}(\boldsymbol x, \boldsymbol z) \boldsymbol{J}_{\boldsymbol{g}^{-1}}(\widehat{\boldsymbol \eta})^\top.
	\end{split}
\end{equation*}
This theoretical pipeline successfully quantifies localized, observation-specific uncertainty for any base distribution $\mathcal{D}$, entirely independent of the non-linear complexity or unstructured nature of the underlying data tensor $\boldsymbol{x}$.

Furthermore, establishing the asymptotic distribution of the native parameters seamlessly facilitates inference on arbitrary transformations of the parameter space. Let $\boldsymbol{f}: \mathbb{R}^p \to \mathbb{R}^q$ be a continuously differentiable vector-valued function representing a specific quantity of interest. Applying the multivariate Delta method once more to the result in \eqref{eq:asymptotic_distribution_phi} yields:
\begin{equation}
	\sqrt{n} \left( f(\widehat{\boldsymbol \phi}) - f( \boldsymbol \phi^{(0)} ) \right) \xlongrightarrow{d} \mathcal{N}_{p}\left(\mathbf{0}, \boldsymbol{\Sigma}_{f}(\boldsymbol x, \boldsymbol z)\right),
	\label{eq:asymptotic_distribution_f_phi}
\end{equation}
where the asymptotic covariance matrix of the transformed parameters is given by:
\begin{equation*}
	\begin{split}
		\boldsymbol \Sigma_f(\boldsymbol x) &= \boldsymbol J_{\boldsymbol f}(\boldsymbol \phi^{(0)}) \boldsymbol \Sigma_{\boldsymbol \phi}(\boldsymbol x, \boldsymbol z) \boldsymbol J_{\boldsymbol f}(\boldsymbol \phi^{(0)})^\top,
		%	&\approx \boldsymbol J_{\boldsymbol f}(\widehat{\boldsymbol \phi}) \boldsymbol \Sigma_{\boldsymbol \phi}(\boldsymbol x, \boldsymbol z) \boldsymbol J_{\boldsymbol f}(\widehat{\boldsymbol \phi})^\top,
	\end{split}
\end{equation*}
with $\boldsymbol{J}_{\boldsymbol f}$ denoting the $q \times p$ Jacobian matrix of $\boldsymbol{f}$. In practice, the empirical covariance estimators $\widehat{\boldsymbol{\Sigma}}_{\boldsymbol{\eta}}$, $\widehat{\boldsymbol{\Sigma}}_{\boldsymbol{\phi}}$ and $\widehat{\boldsymbol{\Sigma}}_{\boldsymbol{f}}$ are obtained by evaluating their respective Jacobian matrices at the maximum likelihood estimates $\widehat{\boldsymbol{\theta}}$, $\widehat{\boldsymbol{\eta}}$ and $\widehat{\boldsymbol{\phi}}$, respectively.

\subsection{GNDR Residual Analysis and Model Selection}
\label{sec:residual_analysis}

In the foundational GAMLSS framework, model selection and goodness-of-fit are traditionally assessed using penalized likelihood metrics, such as the Akaike and Bayesian Information Criteria (AIC and BIC), Global Deviance (GD), and the Generalized AIC \citep{akaike1983information, rigby2005generalized}. To evaluate the adequacy of the assumed response distribution, the framework relies heavily on the normalized quantile residuals introduced by \citet{dunn1996randomized}, which are typically visualized using standard Q-Q plots or worm plots \citep{buuren2001worm}.

However, the direct application of classical penalized likelihood criteria is theoretically problematic within the GNDR framework for two fundamental reasons. First, due to the ultra-high dimensionality of the neural network component and the inherent non-identifiability of the deep weights, the concept of model complexity cannot be captured by a simple parameter count, rendering classical penalty terms invalid. Second, traditional diagnostic statistics inherently evaluate the penalized likelihood exclusively over the training set. In highly flexible, parameterized environments, relying on in-sample evaluation virtually guarantees severe overfitting. Robust model selection must ideally be applied to a strictly disjoint, previously unseen test set. Consequently, we strongly advise against using classical AIC, BIC, or global deviance to evaluate or compare GNDR architectures, except they are comparing the exact same distribution, as it happens in cross validation, for example.

Instead, we adopt normalized quantile residuals evaluated on the hold-out test set as the primary diagnostic tool for assessing goodness-of-fit, as their theoretical validity remains intact regardless of the model's internal structural complexity. This approach relies on the Probability Integral Transform (PIT). Assuming a continuous response $y_i$ is generated from the predicted observation-specific distribution $\mathcal{D}(\boldsymbol{\phi}_i)$ with cumulative distribution function $F_{\mathcal{D}_i}$, the transformed variable $u_i = F_{\mathcal{D}_i}(y_i \mid \widehat{\boldsymbol{\phi}}_i)$ is uniformly distributed on the interval $(0, 1)$. Applying the inverse standard normal cumulative distribution function yields the normalized quantile residuals, $R_i = \Phi^{-1}(u_i)$, which asymptotically follow a standard normal distribution, $\mathcal{N}(0, 1)$, under correct model specification.

In highly flexible, parameterized models, structural misspecifications most frequently manifest as severe deviations in the extreme tails of the residual distribution. To quantify this, we propose utilizing the test statistic from the Anderson-Darling (AD) test as a robust, numerical metric for out-of-sample distributional calibration and model selection. By placing heavier mathematical weight on the tails of the empirical distribution function compared to other standard tests, the AD statistic is highly sensitive to distributional ill-conditioning on the tails. Therefore, when evaluating a set of candidate GNDR architectures or base distributions on the test set, the model minimizing the AD statistic is preferred.

Finally, the flexibility of the GNDR framework naturally accommodates the integration of domain-specific evaluation metrics when appropriate. For instance, in the application of the framework to the TCGA-BRCA survival dataset (Section \ref{sec:applications}), the primary clinical objective is to derive accurate patient-specific survival curves and rank patients by their relative mortality risk. Given the high proportion of right-censored observations in such data, relying on standard normalized quantile residuals is mathematically sub-optimal; evaluating the cumulative distribution function at a censored event time does not yield a true uniform residual, but rather a censored lower bound. In these contexts, survival-specific diagnostics are strictly preferable. Specifically, model calibration can be evaluated using Cox-Snell residuals, which naturally incorporate right-censoring information.
%Furthermore, out-of-sample predictive discrimination can be directly quantified using complementary metrics such as the Inverse Probability of Censoring Weighted (IPCW) concordance index \citep{uno2011c}.

\section{Real-life Applications}
\label{sec:applications}

In this section, we apply the GNDR framework across diverse datasets with distinct statistical structures. By spanning different data modalities and predictive objectives, we demonstrate the broad generalizability and practical efficacy of the proposed methodology.

\subsection{Dutch Boys' Body Mass Index}

To benchmark the performance of the GNDR framework against the traditional GAMLSS architecture, we consider a standard dataset for distributional regression widely utilized within the \texttt{gamlss} R ecosystem. Specifically, we analyze the cross-sectional Body Mass Index (BMI) and age of 7,294 Dutch boys. This dataset serves as a rigorous foundational test for modeling highly non-linear, heteroscedastic growth dynamics. As our objective here is to compare the results with the GAMLSS foundation, we do not split the data into train, validation and test sets, considering only a single batch.

We assume the BMI of the $i$-th boy, $y_i$, follows a Box-Cox $t$ (BCT) distribution, denoted as $\text{BCT}(\mu_i, \sigma_i, \nu_i, \tau_i)$. The modern classical GAMLSS specification represents the distributional parameters as additive smooth functions of age:
\begin{equation*}
	\begin{cases}
		\mu_i &= h_1(x_i),\\
		\log( \sigma_i ) &= h_2(x_i),\\
		\nu_i &= h_3(x_i),\\
		\log(\tau_i) &= h_4(x_i),
	\end{cases}
\end{equation*}
where $h_k(\cdot)$ represents an arbitrary smoothing function for $k \in \{1, 2, 3, 4\}$, and $x_i$ is the observed age. For our classical baseline, we define $h_k$ using penalized B-splines, denoted as $\text{pb}(\cdot)$.

Under the proposed GNDR framework, the native parameter vector is $\boldsymbol \phi = (\mu, \sigma, \nu, \tau)^\top$. We restrict the structured classical components by setting $\mathcal{G}_k \equiv 0$ for all $k$, allowing an arbitrary neural network architecture $\mathcal{F}$ to process the raw age and output the four predictive signals. Following the additive structure outlined in Equation (\ref{eq:gndr_parameters_structure}), the GNDR model is formulated as:
\begin{equation*}
	\begin{cases}
		\mu_i &= \mathcal{F}_1(x_i; \boldsymbol \omega),\\
		\log(\sigma_i) &= \mathcal{F}_2(x_i; \boldsymbol \omega),\\
		\nu_i &= \mathcal{F}_3(x_i; \boldsymbol \omega),\\
		\log(\tau_i) &= \mathcal{F}_4(x_i; \boldsymbol \omega).
	\end{cases}
\end{equation*}
Leveraging the universal approximation capabilities of deep neural networks, there is no requirement to manually pre-process or apply fractional power transformations to the age variable to capture sharp growth trajectories. We define $\mathcal{F}$ as a feed-forward architecture comprising two hidden layers. The first layer contains 24 neurons, while the second acts as an aggressive representational bottleneck with only 2 neurons, both with Gaussian Error Linear Unit (GELU) activation functions. As established in Section \ref{sec:asymptotic_theory}, restricting the dimension of the last hidden layer implicitly regularizes the model, ensuring the statistical degrees of freedom are preserved for stable variance estimation. The final output layer maps to 4 nodes with linear activations, corresponding to the unconstrained parameters of the BCT distribution.

The network weights $\boldsymbol \omega$ are optimized to maximize the conditional log-likelihood function of the sample, defined analytically as:
\begin{equation*}
	\ell(\boldsymbol \omega \mid \boldsymbol y, \boldsymbol x) = \sum_{i=1}^n \left[ (\nu_i-1) \log y_i - \nu_i \log \mu_i - \log \sigma_i + \log f_T(z_i) - \log F_T\left( (\sigma_i|\nu_i|)^{-1} \right) \right],
\end{equation*}
where $f_T(\cdot)$ and $F_T(\cdot)$ are the probability density function and cumulative distribution function, respectively, of a standard Student's $t$-distribution with $\tau_i$ degrees of freedom, and the scaled variable $z_i$ is defined piecewise as:
\begin{equation*}
	z_i = \begin{cases}
		\frac{1}{\sigma_i \nu_i}\left[\left(\frac{y_i}{\mu_i}\right)^{\nu_i}-1\right], & \text{ if } \nu_i \neq 0,\\
		\frac{1}{\sigma_i} \log \left(\frac{y_i}{\mu_i}\right), & \text{ if } \nu_i = 0.
	\end{cases}
\end{equation*}

While the GNDR framework natively leverages the automatic differentiation (autodiff) capabilities of the TensorFlow backend, calculating the gradients of special functions like $F_T$ via numerical approximations can introduce computational bottlenecks or numerical instability during stochastic optimization. To ensure robust and efficient convergence, we override the standard autograd tape in this specific application. Using the \texttt{@tf.custom\_gradient} decorator, we explicitly inject the exact analytical first- and second-order derivatives for the BCT distribution derived by \citet{rigby2006using}.

%Figure \ref{fig:dutch_boys_gamlss_thetaflow} compares the predicted parameter curves as a function of age. The trajectories for the location ($\mu$), scale ($\sigma$), and skewness ($\nu$) match closely between both models. For the degrees of freedom ($\tau$), both models exhibit distinct curves; however, this discrepancy does not indicate that one model is strictly superior. Rather, it highlights the well-documented challenge of estimating the degrees of freedom in a Student's $t$-distribution due to flat likelihood surfaces. Because all estimated values for $\tau$ exceed 10, variations in this parameter produce practically negligible differences in the resulting distribution, as the $t$-distribution naturally converges to a Gaussian-like kurtosis.

Figure \ref{fig:dutch_boys_gamlss_thetaflow} compares the predicted parameter trajectories as a continuous function of age, along with the predicted confidence bounds for both models. The curves for the location ($\mu$), scale ($\sigma$), and skewness ($\nu$) exhibit near-perfect alignment between the classical GAMLSS baseline and the neural-augmented GNDR model. For the kurtosis parameter, $\tau$, the models produce distinct trajectories. However, this discrepancy is not indicative of structural misspecification in either approach. Rather, it reflects the known phenomenon of flat log-likelihood surfaces when estimating the degrees of freedom in $t$-family distributions. Because all estimated values for $\tau$ strictly exceed 10 across the entire age domain, variations within this parameter space produce practically negligible differences in the resulting density, as the heavy tails have effectively converged to Gaussian kurtosis.

\begin{figure}[htbp]
	\centering
	\includegraphics[width=0.9\linewidth]{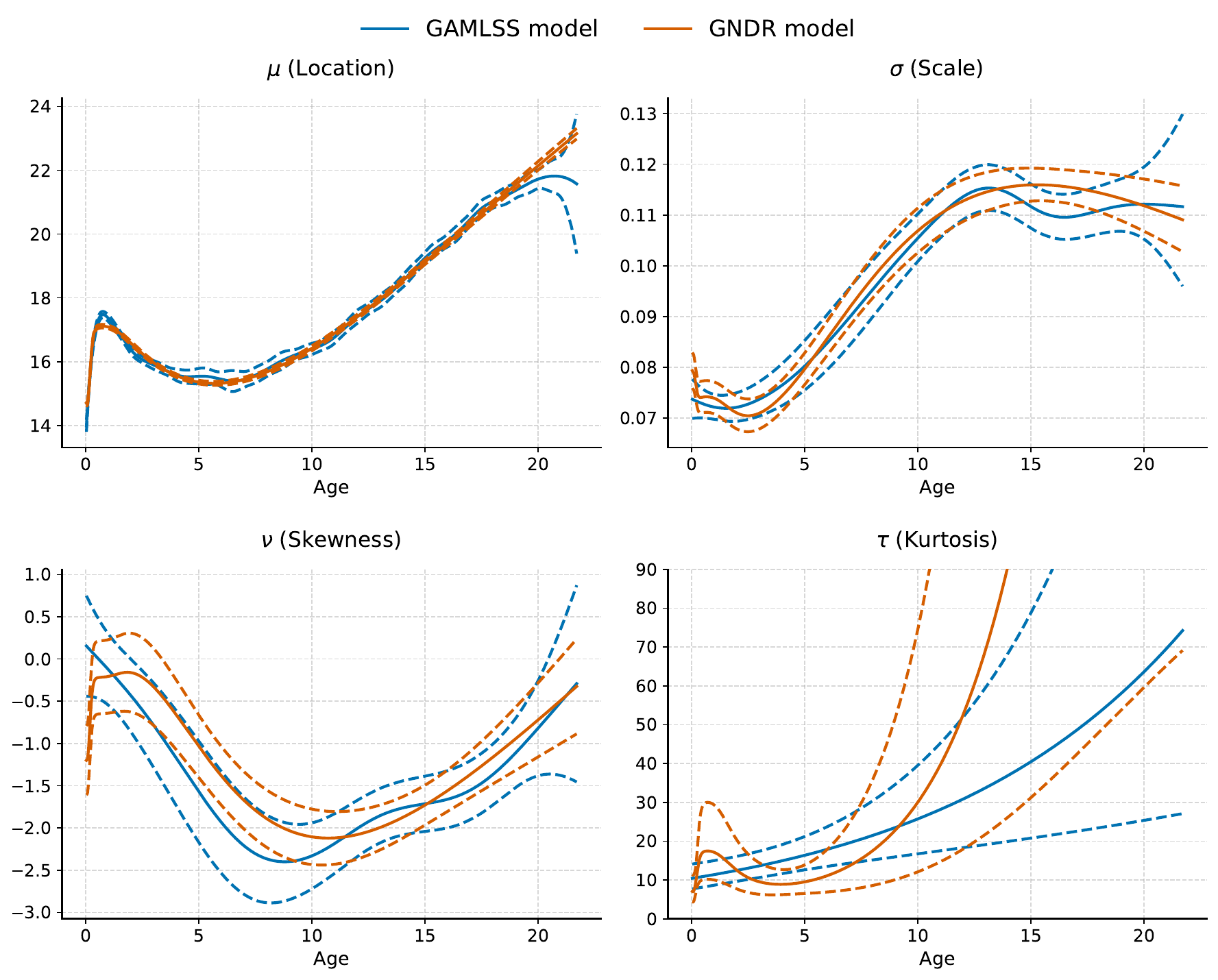}
	\caption{BCT parameters against age, comparing the classical penalized spline GAMLSS predictions with the neural-augmented GNDR framework.}
	\label{fig:dutch_boys_gamlss_thetaflow}
\end{figure}

%Figure \ref{fig:bmi_qqplot_wormplot} compares the traditional Q-Q plots and worm plots for both the GNDR and GAMLSS models, utilizing Gaussian quantile residuals. As anticipated, both models achieve a highly satisfactory diagnostic pattern, with the vast majority of residuals falling well within the confidence envelopes of the worm plots. 

Figure \ref{fig:bmi_qqplot_wormplot} presents the traditional Q-Q plots and worm plots for both models, evaluated using normalized quantile residuals. As anticipated, both frameworks achieve a highly satisfactory diagnostic fit, with the vast majority of residuals safely contained within the 95\% confidence envelopes of the worm plots. This confirms that the unconstrained neural architecture $\mathcal{F}$ successfully learned a feature representation that matches the rigorously tuned penalized splines of the classical baseline, achieving exceptional goodness-of-fit.

\begin{figure}[htbp]
	\centering
	\includegraphics[width=0.9\linewidth]{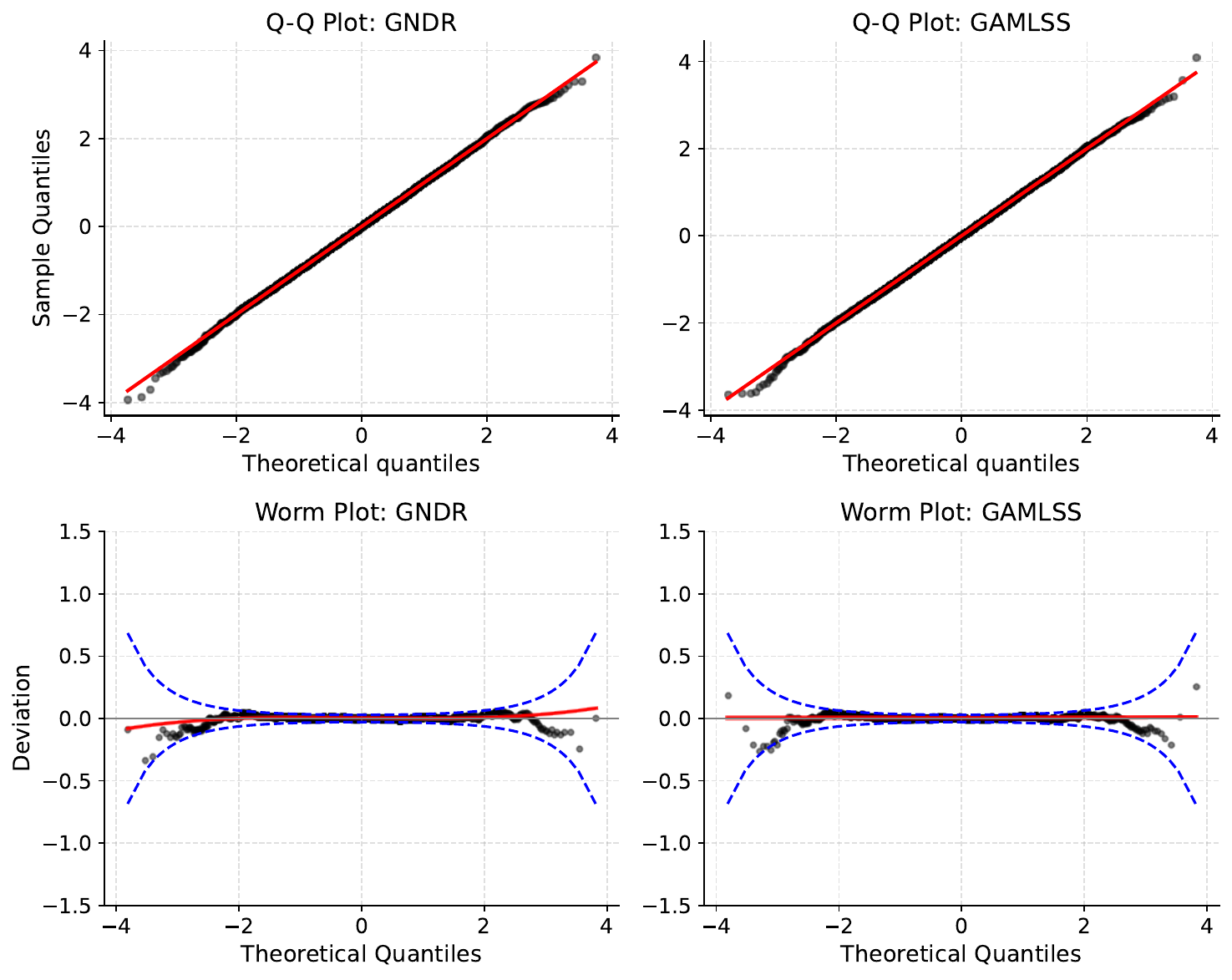}
	\caption{Q-Q plots and worm plots for both the GNDR and GAMLSS models for the Dutch Boys' Body Mass Index dataset.}
	\label{fig:bmi_qqplot_wormplot}
\end{figure}

\subsection{Hospital stay data}

%Also under, \citet{rigby2005generalized}, the authors consider the 1,383 observations from a study at the Hospital del Mar in Barcelona from 1988 to 1990 from \citet{gange et al}. Considering a Beta-Binomial distribution for the number of inappropriate days (noinap) out of the total number of days (los) each patient spent in the hospital, the authors use the GAMLSS flexible structure to compare different combinations of the variables ``age'', ``ward'' ``year'' and ``los'' into the parameters for location and scale of the Beta-Binomial distribution. Here, we select the two best models from their paper (III and IV), which we denote here as (a) and (b). Their structure, alongside our GNDR-based proposals are displayed in Table \ref{tab:gamlss_hospital_stay_models}. As usual, $\text{cs}$ represent a cubic spline. For the GNDR models, by including the sum of varaibles inside a $\mathcal{F}_k$ function, we imply that these variables should be concatenated into a single tensor, $\boldsymbol x$, which in this case will be a simple design matrix. For variables summed outside a $\mathcal{F}_k$, we automatically assume a standard linear model.

We consider the hospital stay dataset originally analyzed by \citet{gange1996use} and subsequently modeled by \citet{rigby2005generalized}. The dataset comprises 1,383 observations from patients at the Hospital del Mar in Barcelona between 1988 and 1990. The objective is to model the number of inappropriate days (\texttt{noinap}) out of the total length of stay (\texttt{los}) each patient spent in the hospital. Because the response variable represents overdispersed binomial counts, it is assumed to follow a Beta-Binomial distribution, parameterized by its location $\mu$ (the mean probability of an inappropriate day) and scale $\sigma$ (the dispersion parameter).

In their foundational work, \citet{rigby2005generalized} utilized the flexible GAMLSS structure to evaluate various combinations of the available predictors: \texttt{age}, \texttt{ward}, \texttt{year}, and \texttt{loglos}, the log-transformed length of stay on both $\mu$ and $\sigma$. We select the two best-performing classical models from their study (Models III and IV), which we denote here as models (a) and (b), respectively. To compare these classical approaches with the GNDR framework, we formulate three alternative neural architectures: (c), (d) and (e).

The predictor structures for all five models are summarized in Table \ref{tab:gamlss_hospital_stay_models}. For classical components, $\text{cs}(\cdot)$ denotes a cubic smoothing spline, and the $+$ operator denotes the standard linear combination of basis vectors. For the neural components, including a sum of variables inside the network mapping $\mathcal{F}_k(\cdot)$ indicates that these variables are column-wise concatenated to form the unstructured input tensor $\boldsymbol{x}$. The chosen architecture for $\mathcal{F}$ in all models was two hidden dense layers with 24 and 1 neuron, respectively, both with GELU activation functions and followed by the final, unconstrained output for the effect on $\text{logit}(\mu)$. We note that for all models, the intercept term is considered for both $\text{logit}(\mu)$ and $\log(\sigma)$.

\begin{table}[]
	\centering
	\caption{Predictor structures for the Beta-Binomial distribution parameters.}
	{\small
	\begin{tabular}{c|c|c}
		\toprule
		Model & Link & Terms\\
		\midrule
		\multirow{2}{*}{(a)} & $\text{logit}(\mu)$ & $\text{year} + \text{ward} + \text{cs}(\text{loglos}, 1)$\\
		& $\log(\sigma)$ & $\text{year} + \text{ward}$\\
		\midrule
		\multirow{2}{*}{(b)} & $\text{logit}(\mu)$ & $\text{year} + \text{ward} + \text{cs}(\text{loglos}, 1) + \text{cs}(\text{age}, 1)$\\
		& $\log(\sigma)$ & $\text{year} + \text{ward}$\\
		\midrule
		\multirow{2}{*}{(c)} & $\text{logit}(\mu)$ & $\text{year} + \text{ward} + \mathcal{F}_1(\text{loglos})$\\
		& $\log(\sigma)$ & $\text{year} + \text{ward}$\\
		\midrule
		\multirow{2}{*}{(d)} & $\text{logit}(\mu)$ & $\text{year} + \text{ward} + \mathcal{F}_1(\text{loglos} + \text{age})$\\
		& $\log(\sigma)$ & $\text{year} + \text{ward}$\\
		\midrule
		\multirow{2}{*}{(e)} & $\text{logit}(\mu)$ & $\mathcal{F}_1\left( \text{year} + \text{ward} + \text{loglos} + \text{age}\right)$\\
		& $\log(\sigma)$ & $\mathcal{F}_2\left( \text{year} + \text{ward} + \text{loglos} + \text{age}\right)$\\
		\bottomrule
	\end{tabular}
	}
	\label{tab:gamlss_hospital_stay_models}
\end{table}

%For instance, model (d) can be written as in Equation (\ref{eq:gndr_parameters_linear_structure}). We have
%\begin{align*}
%	\mathcal{G}_1(\boldsymbol z_1; \boldsymbol \beta_1) = \boldsymbol z_1^\top \boldsymbol \beta_1, &\quad \mathcal{F}_1(\boldsymbol x; \boldsymbol \omega): \text{arbitrary neural network},\\
%	\mathcal{G}_2(\boldsymbol z_2; \boldsymbol \beta_2) = \boldsymbol z_2^\top \boldsymbol \beta_2, &\quad \mathcal{F}_2(\boldsymbol x; \boldsymbol \omega) \equiv 0.
%\end{align*}
%where $\boldsymbol z_1 = \boldsymbol z_2$ represent the design vectors corresponding to variables ``year'' and ``ward'' and $\boldsymbol x$ represent the 2-dimensional numerical vector corresponding to ``loglos'' and ``age''. Model (e) in contrast, represent the most flexible setup we can obtain, by assuming both parameters to be modeled as outputs from a neural network, which gets as inputs all variables at once. We note that out framework assumes a single neural network structure for the entire model, therefore, it still does not admit a structure like $\mathcal{F}_1(year + loglos)$ with $\mathcal{F}_2(year)$.

To formalize this mathematically, consider model (d). Following the semi-parametric structure defined in Equation (\ref{eq:gndr_parameters_linear_structure}), the classical and neural components map as follows:
\begin{align*}
	\mathcal{G}_1(\boldsymbol z_1; \boldsymbol \beta_1) &= \boldsymbol z_1^\top \boldsymbol \beta_1, &\quad \mathcal{F}_1(\boldsymbol x; \boldsymbol \omega) &: \text{arbitrary neural network},\\
	\mathcal{G}_2(\boldsymbol z_2; \boldsymbol \beta_2) &= \boldsymbol z_2^\top \boldsymbol \beta_2, &\quad \mathcal{F}_2(\boldsymbol x; \boldsymbol \omega) &\equiv 0,
\end{align*}
where $\boldsymbol z_1 = \boldsymbol z_2$ represents the shared classical design vector encoding the categorical predictors \texttt{year} and \texttt{ward}, and $\boldsymbol{x}$ represents the 2-dimensional numerical matrix containing \texttt{loglos} and \texttt{age}.

In contrast, model (e) represents the most flexible, fully non-linear architecture available within the framework. By assigning $\mathcal{G}_1 \equiv \mathcal{G}_2 \equiv 0$, both distributional parameters are modeled directly as terminal outputs from the neural network, which receives the concatenation of all available covariates as its input tensor $\boldsymbol{x}$. It is important to note an architectural constraint of the current GNDR implementation: because $\mathcal{F}$ relies on a shared set of deep hidden layers ($\boldsymbol{\omega}_{\text{deep}}$) for joint representation learning, the input tensor $\boldsymbol{x}$ is globally defined for the entire network. Consequently, the framework does not natively support parameter-specific feature masking within the neural component; for example, it cannot simultaneously specify $\mathcal{F}_1(\text{year} + \text{loglos})$ and $\mathcal{F}_2(\text{year})$ unless a multi-branch network architecture is explicitly engineered.

%Since our interest in this application is to select the best predicting model, we split the data in train and test sets of sizes 1,106 and 277, respectively. During training, the first set was also split to accomodate a 221 observations validation set, for early stopping and learning rate decay control. The AD statistics for residual normality of each model are disposed in Table \ref{tab:hospital_stay_ad_values} for both the train and test sets. The selected model was (d). In fact, this model corresponds to an extension of model (b), which was the winning model in its gamlss application according to the generalized deviance. However, instead of individual splines, a bivariate, neural network surface is obtained for the variables ``loglos'' and ``age''. That surface is given if Figure \ref{fig:hospital_stay_surface}. As the graph suggests, the small decrease in $\mu$ for higher age values reported in \citet{rigby2005generalized} corresponds only to higher periods of stay. For lower periods of stay, we can see a clear, increasing linear pattern for $\mu$ with respect to age. Figure \ref{fig:hospital_stay_wormplots} shows the wormplots for train and test sets in both (b) and (d) models. As it can be seen, model (d) clearly shows a much better fit, especially in the test set.

Because our primary objective in this application is out-of-sample model selection and validation, we partitioned the 1,383 observations into a training set ($1,106$) and a strictly isolated hold-out test set ($277$). During the initial optimization phase, the training set was further partitioned to sequester a 221-observation validation set, which was utilized strictly to govern early stopping and dynamic learning rate decay.

The Anderson-Darling (AD) test statistics for the normalized quantile residuals, evaluated on both the training and test sets, are presented in Table \ref{tab:hospital_stay_ad_values}. Based on the out-of-sample evaluation established in Section \ref{sec:residual_analysis}, model (d) emerges as the superior architecture, minimizing the AD statistic on the test set. Crucially, model (d) serves as a direct neural extension of model (b), which was the optimal architecture identified via penalized deviance in the original GAMLSS analysis. However, rather than relying on strictly additive, independent smoothing splines for \texttt{loglos} and \texttt{age}, the unconstrained neural component $\mathcal{F}_1$ naturally learns a bivariate, non-linear interaction surface, which is given in Figure \ref{fig:hospital_stay_surface}. As the surface topography reveals, the marginal decrease in $\mu$ for older patients, originally reported as a global main effect by \citet{rigby2005generalized}, is actually a conditional effect restricted to longer periods of stay. Conversely, for shorter hospitalizations, the model captures a distinctly increasing linear relationship between age and the probability of an inappropriate day.

\begin{table}[]
	\centering
	\caption{Anderson-Darling statistics for all proposed models in the Hospital Stay dataset, evaluated via normalized quantile residuals.}
	{\small
		\begin{tabular}{c|c|c}
			\toprule
			Dataset & Model & AD Statistic\\
			\midrule
			\multirow{5}{*}{Train} & (d) & 1.0621\\
			& (e) & 1.1746\\
			& (a) & 1.5819\\
			& (c) & 1.7100\\
			& (b) & 2.0453\\
			\midrule
			\multirow{5}{*}{Test} & (d) & 0.1585\\
			& (e) & 0.1793\\
			& (c) & 0.2406\\
			& (a) & 0.4101\\
			& (b) & 0.6846\\
			\bottomrule
		\end{tabular}
	}
	\label{tab:hospital_stay_ad_values}
\end{table}

\begin{figure}[htbp]
	\centering
	\includegraphics[width=0.9\linewidth]{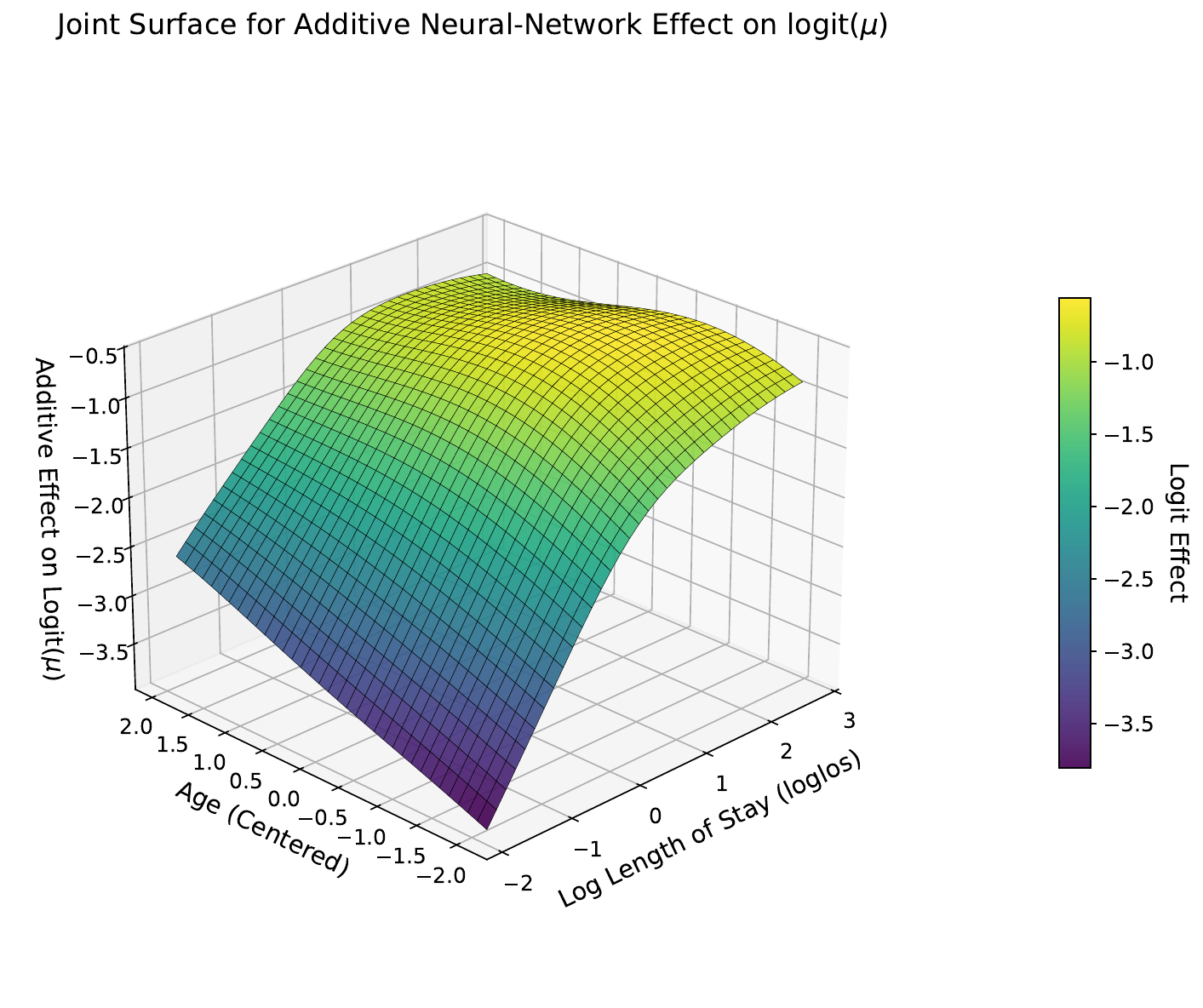}
	\caption{Surface plot illustrating the learned non-linear neural interaction effect ($\mathcal{F}_1$) of \texttt{loglos} and \texttt{age} on $\text{logit}(\mu)$.}
	\label{fig:hospital_stay_surface}
\end{figure}

Figure \ref{fig:hospital_stay_wormplots} compares the worm plots for both the training and test sets between the optimal classical GAMLSS model (b) and the optimal GNDR model (d). As evidenced by the tighter adherence to the horizontal origin, model (d) achieves substantially better distributional calibration. This improvement is particularly pronounced in the hold-out test set, demonstrating that the neural interaction surface provides a superior, generalizable fit without succumbing to overfitting.

\begin{figure}[htbp]
	\centering
	\includegraphics[width=0.9\linewidth]{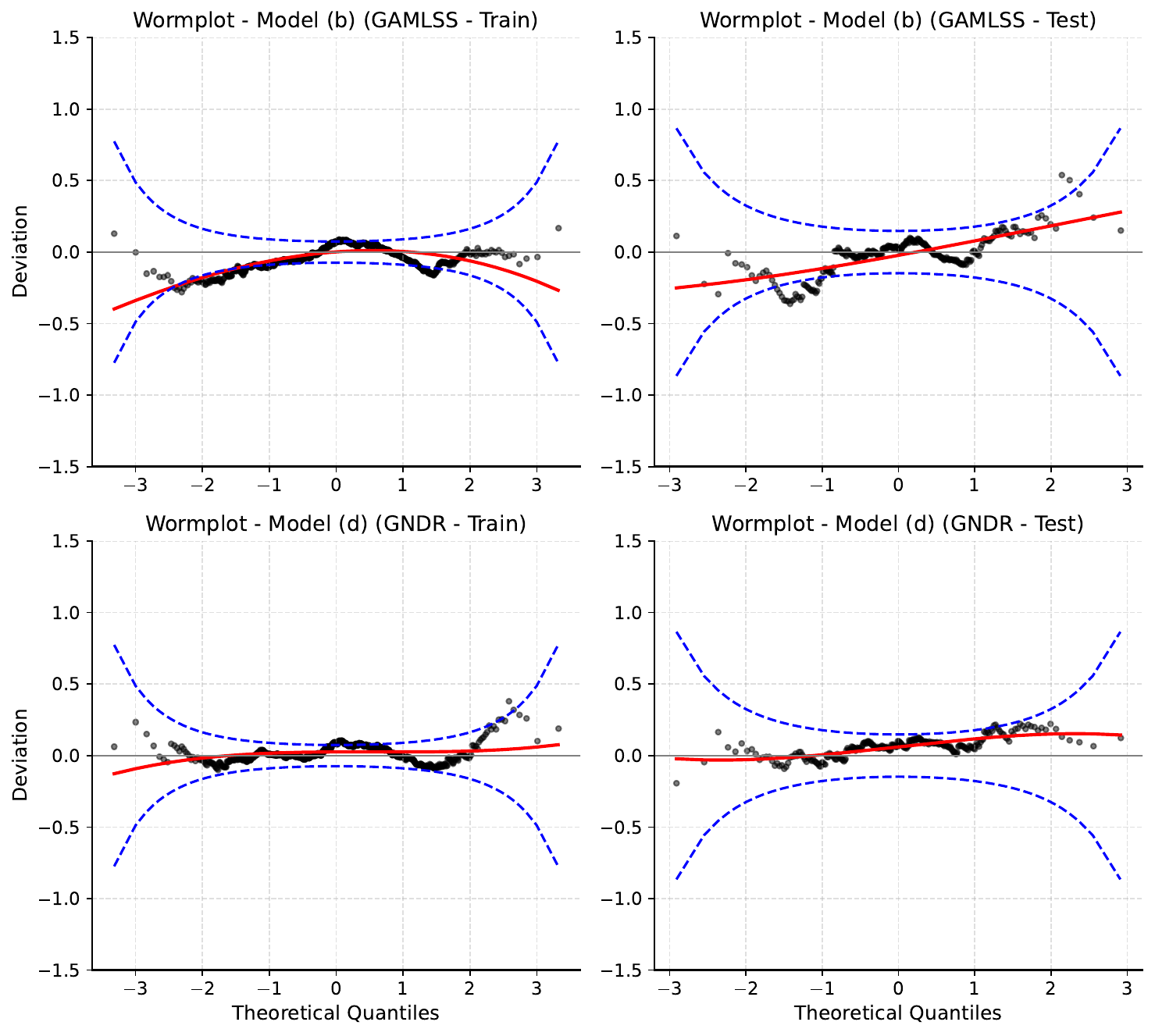}
	\caption{Wormplots for train and test sets for both models (b) and (d) in the Hospital Stay dataset.}
	\label{fig:hospital_stay_wormplots}
\end{figure}

\subsection{Breast Cancer Gene Expression Data (TCGA-BRCA)}
\label{sec:tcga_brca_application}

Genomic studies currently represent the state-of-the-art in oncology research. Presently, numerous complex gene pathways have been identified as direct drivers of breast cancer pathogenesis, guided by foundational genomic databases such as the Kyoto Encyclopedia of Genes and Genomes (KEGG) \citep{kanehisa2000kegg} and further expanded by recent discoveries of novel molecular drivers and metastasis-associated signatures \citep{samad2025breast, yao2026identification}. Among the diverse analytical approaches to high-dimensional omics data, which now rapidly integrate beyond transcriptomics and genomics to include epigenomic, proteomic, and metabolomic layers \citep{adeyemo2026clinical, hemme2026challenges},  gene expression profiling fundamentally relies on analyzing large-scale count matrices derived from multiple patient biopsies.

In this application, we analyze breast cancer gene expression (GE) data from The Cancer Genome Atlas (TCGA-BRCA). The dataset comprises a high-dimensional feature matrix capturing the expression levels of 17,235 selected genes across a cohort of $n = 1,071$ female patients. For the $i$-th patient, the clinical response is defined by their observed survival time and right-censoring indicator, denoted as $(y_i, \delta_i)$. Notably, the cohort exhibits heavy right-censoring, which accounts for 86.09\% of the observations. A similar iteration of this dataset was analyzed by \citet{ching2018cox} to develop Cox-nnet, an architecture that directly integrates high-dimensional GE data into the semi-parametric Cox proportional hazards model \citep{cox1972regression}. While Cox-nnet effectively bridges deep learning with survival analysis through partial likelihood maximization, our objective in this application is to demonstrate the GNDR framework's capacity for fully parametric, non-linear survival modeling. We note that, to our knowledge, no other works were able to analyse highly dimensional gene expression data directly into a fully parametric survival model, as allowed by the GNDR framework.

Drawing structural inspiration from Cox-nnet, we design a specific neural network architecture $\mathcal{F}(\boldsymbol x ; \boldsymbol \omega)$ to process the high-dimensional gene count vector $\boldsymbol x$. The input tensor is fed into an initial hidden layer comprising 32 neurons with hyperbolic tangent (\texttt{tanh}) activations, regularized via an Elastic Net penalty \citep{zou2005regularization} on the weights. This is followed by a Dropout layer and a subsequent dense hidden layer containing two neurons equipped with \texttt{softplus} activations. The terminal dense layer yields either one or two outputs, depending on the specific parametric base distribution being fitted. Compared to the 132-neuron width proposed in the original Cox-nnet paper, we opted for a significantly more aggressive bottleneck in the first hidden layer (32 neurons). Our empirical diagnostics revealed that this extreme dimensionality reduction is necessary to ensure stable maximum likelihood convergence within a fully parametric survival context.

To optimally tune the $L_1$ and $L_2$ penalty weights of the Elastic Net alongside the dropout rate, we employed 5-fold cross-validation over a discrete hyperparameter grid, evaluated strictly on the training set. Because the hyperparameter optimization is conducted internally within each fixed distributional family, model selection was governed by the maximization of the out-of-fold log-likelihood.

In this application, we evaluate the Weibull, Log-normal, and Log-logistic distributions, alongside their respective extensions under a standard mixture cure model framework \citep{berkson1952survival}. The mixture cure model accommodates survival scenarios where a subset of patients has a positive probability of becoming long-term survivors, meaning they are effectively cured and no longer susceptible to the event of interest. This approach naturally introduces a patient-specific cure probability, $p_i$, as an additional distributional parameter. Under this mixture framework, the marginal survival function for the $i$-th patient is formulated as:
\begin{equation*}
	S_i(t) = p_i + (1 - p_i)S_{0i}(t),
\end{equation*}
where $S_{0i}(t)$ denotes the proper survival function for the susceptible population. Crucially, $S_{0i}(t)$ is also patient-specific; it is analytically evaluated using the remaining base distribution parameters from the predicted native vector $\boldsymbol{\phi}_i$ (e.g., the Weibull scale and shape parameters). The complete set of candidate models evaluated in this application, along with their neural network predictor structures and cross-validated hyperparameters, is detailed in Table \ref{tab:tcga_models}. We note that the structured term $\xi_1$ restricts the corresponding parameter to a global constant across all patients, whereas the neural component $\mathcal{F}$ dynamically generates patient-specific parameters driven by $\boldsymbol x$, i.e. the individual's high-dimensional gene expression profile.

A critical structural challenge arises specifically within the Cure Mixture Weibull model. Because the proper Weibull survival function, $S_0(t) = \exp(-(t/\lambda)^k)$, approaches $1$ for all $t$ as the scale parameter $\lambda \to \infty$, the unconstrained neural network can artificially emulate a high cure fraction simply by driving $\lambda$ toward infinity, effectively stretching the susceptible survival curve into a flat line. This severe topological confounding renders the cure fraction $p$ and the scale $\lambda$ weakly identified, routinely causing optimization algorithms to degenerate (i.e., exploding $\lambda$ and forcing $p \to 0$). To guarantee strict parameter identifiability and force the network to explicitly utilize the cure fraction, we rigorously restrict the Weibull scale parameter to an empirically plausible biological upper bound. As indicated in Table \ref{tab:tcga_models}, we enforce $\lambda \in (0, 40)$ by projecting the neural output through a scaled-logit link function, $\text{logit}(\lambda / 40)$.

\begin{table}[]
	\centering
	\caption{Survival models and their respective predictor structures for the TCGA-BRCA dataset.}
	{\small
		\begin{tabular}{c|c|c|c|ccc}
			\toprule
			Model & Parameter & link & Terms & $L_1$ & $L_2$ & Dropout\\
			\midrule
			\multirow{2}{*}{Weibull} & shape & $\log(k)$ & $\xi_1$ & \multirow{2}{*}{0.0005} & \multirow{2}{*}{0.0001} & \multirow{2}{*}{0.3}\\
			& scale & $\text{log}(\lambda)$ & $\mathcal{F}_1(\boldsymbol x ; \boldsymbol \omega)$ & & &\\
			\midrule
			\multirow{2}{*}{Log-normal} & log-scale & $\log(\sigma)$ & $\xi_1$ & \multirow{2}{*}{0.0001} & \multirow{2}{*}{0.005} & \multirow{2}{*}{0.1}\\
			& log-location & $\mu$ & $\mathcal{F}_1(\boldsymbol x ; \boldsymbol \omega)$ & & &\\
			\midrule
			\multirow{2}{*}{Log-logistic} & shape & $\log(\beta)$ & $\xi_1$ & \multirow{2}{*}{0.0} & \multirow{2}{*}{0.0005
			} & \multirow{2}{*}{0.1}\\
			& scale & $\log(\alpha)$ & $\mathcal{F}_1(\boldsymbol x ; \boldsymbol \omega)$ & & &\\
			\midrule
			\multirow{3}{*}{Cure Mixture Weibull} & shape & $\log(k)$ & $\xi_1$ & \multirow{3}{*}{0.1} & \multirow{3}{*}{0.01} & \multirow{3}{*}{0.45}\\
			& scale & $\text{logit}(\lambda / 40)$ & $\mathcal{F}_1(\boldsymbol x ; \boldsymbol \omega)$ & & &\\
			& cure probability & $\text{logit}(p)$ & $\mathcal{F}_2(\boldsymbol x ; \boldsymbol \omega)$ & & &\\
			\midrule
			\multirow{3}{*}{Cure Mixture Log-normal} & log-scale & $\log(\sigma)$ & $\xi_1$ & \multirow{3}{*}{0.0005} & \multirow{3}{*}{0.0001} & \multirow{3}{*}{0.15}\\
			& log-location & $\mu$ & $\mathcal{F}_1(\boldsymbol x ; \boldsymbol \omega)$ & & &\\
			& cure probability & $\text{logit}(p)$ & $\mathcal{F}_2(\boldsymbol x ; \boldsymbol \omega)$ & & &\\
			\midrule
			\multirow{3}{*}{Cure Mixture Log-logistic} & shape & $\log(\beta)$ & $\xi_1$ & \multirow{3}{*}{0.0001} & \multirow{3}{*}{0.005} & \multirow{3}{*}{0.1}\\
			& scale & $\log(\alpha)$ & $\mathcal{F}_1(\boldsymbol x ; \boldsymbol \omega)$ & & &\\
			& cure probability & $\text{logit}(p)$ & $\mathcal{F}_2(\boldsymbol x ; \boldsymbol \omega)$ & & &\\
			\bottomrule
		\end{tabular}
	}
	\label{tab:tcga_models}
\end{table}

Given the high dimensionality of the parameter space, the heterogeneity of the cross-validated hyperparameters, and the intrinsic geometry of each baseline distribution, we reiterate from Section \ref{sec:residual_analysis} that penalized log-likelihood metrics are theoretically invalid for direct structural comparison across these models. Furthermore, to rigorously account for the right-censored lifetimes, we compute randomized normalized quantile residuals \citep{li2021model}. Let $F(\cdot \mid \widehat{\boldsymbol \phi}_i)$ denote the predicted cumulative distribution function for the $i$-th patient based on the models defined in Table \ref{tab:tcga_models}. The randomized residual is defined as:
\begin{equation}
	R_i = \begin{cases}
		\Phi^{-1}\left( F(y_i \mid \widehat{\boldsymbol \phi}_i ) \right), & \text{ if } \delta_i = 1,\\
		\Phi^{-1} \left( U_i \right), & \text{ if } \delta_i = 0,
	\end{cases}
	\label{eq:resid_quantile_censors}
\end{equation}
where $U_i \sim \text{Uniform}( F(y_i \mid \widehat{\boldsymbol \phi}_i ), 1 )$ and $\Phi^{-1}$ is the standard normal quantile function. By drawing from this truncated uniform distribution, the procedure properly acknowledges that the unobserved true event time corresponds to a cumulative probability strictly greater than the probability evaluated at the censoring time.

While this randomization provides a theoretically consistent extension for goodness-of-fit assessment in survival models, we emphasize that under extreme right-censoring, this metric can artificially favor naive models. If a misspecified model assigns a spuriously high survival probability (i.e., $F(y_i \mid \widehat{\boldsymbol \phi}_i) \approx 0$) to the vast majority of censored patients ($\delta_i = 0$), the residuals will be sampled almost directly from the unconstrained standard normal generator, producing an illusion of perfect fit. To mitigate this vulnerability, it is strictly recommended to pair the numerical residual analysis with a visual inspection of the individual survival curves, ensuring that the model's predictions do not degenerate into uninformative, flatline survival trajectories.

Within the specialized domain of survival analysis, predictive accuracy is predominantly evaluated using metrics such as Harrell's C-index \citep{harrell1982evaluating}, the IPCW C-index \citep{uno2011c}, and the Integrated Brier Score \citep{graf1999assessment}. While these metrics are invaluable for assessing clinical discrimination and pointwise survival calibration, they fundamentally evaluate the model merely as a predictive engine. The core objective of the GNDR framework, however, is to capture the complete topological structure of the true data-generating distribution. Because our primary interest lies in verifying global distributional fidelity, we maintain the empirical Anderson-Darling (AD) statistic on the randomized residuals as our definitive comparative criterion, supported by the aforementioned cautionary visual verifications. Table \ref{tab:tcga_models_comparison} presents the AD test statistic for each candidate architecture, identifying the Log-logistic distribution as the optimal fit.

\begin{table}[]
	\centering
	\caption{Anderson-Darling statistics for all proposed models in the TCGA-BRCA dataset, evaluated via normalized quantile residuals.}
	{\small
		\begin{tabular}{c|l|c}
			\toprule
			Dataset & Model & AD Statistic\\
			\midrule
			\multirow{6}{*}{Train} & Log-logistic & 0.1613\\
			& Mixture Weibull & 0.2028\\
			& Weibull & 0.2487\\
			& Mixture Log-logistic & 0.2501\\
			& Log-normal & 1.6964\\
			& Mixture Log-normal & 1.7643\\
			\midrule
			\multirow{6}{*}{Test} & Log-logistic & 0.3092\\
			& Weibull & 0.3181\\
			& Mixture Log-logistic & 0.4675\\
			& Mixture Log-normal & 0.5594	\\
			& Mixture Weibull & 0.5659\\
			& Log-normal & 0.7375\\
			\bottomrule
		\end{tabular}
	}
	\label{tab:tcga_models_comparison}
\end{table}

Given that the Log-logistic distribution was selected via out-of-sample distributional calibration, Figure \ref{fig:loglogistic_kaplan_meier} displays the predicted patient-specific survival curves juxtaposed against the overall Kaplan-Meier estimator for the cohort. As expected from a well-calibrated model, the marginal average of the predicted survival curves tracks the empirical survival curve closely. Furthermore, this visual inspection confirms that the individual predicted curves do not degenerate into flatline trajectories. This effectively validates our diagnostic strategy, ensuring that the randomized normalized quantile residuals provide a robust goodness-of-fit assessment rather than an artifact of naive extrapolation.

\begin{figure}[htbp]
	\centering
	\includegraphics[width=0.9\linewidth]{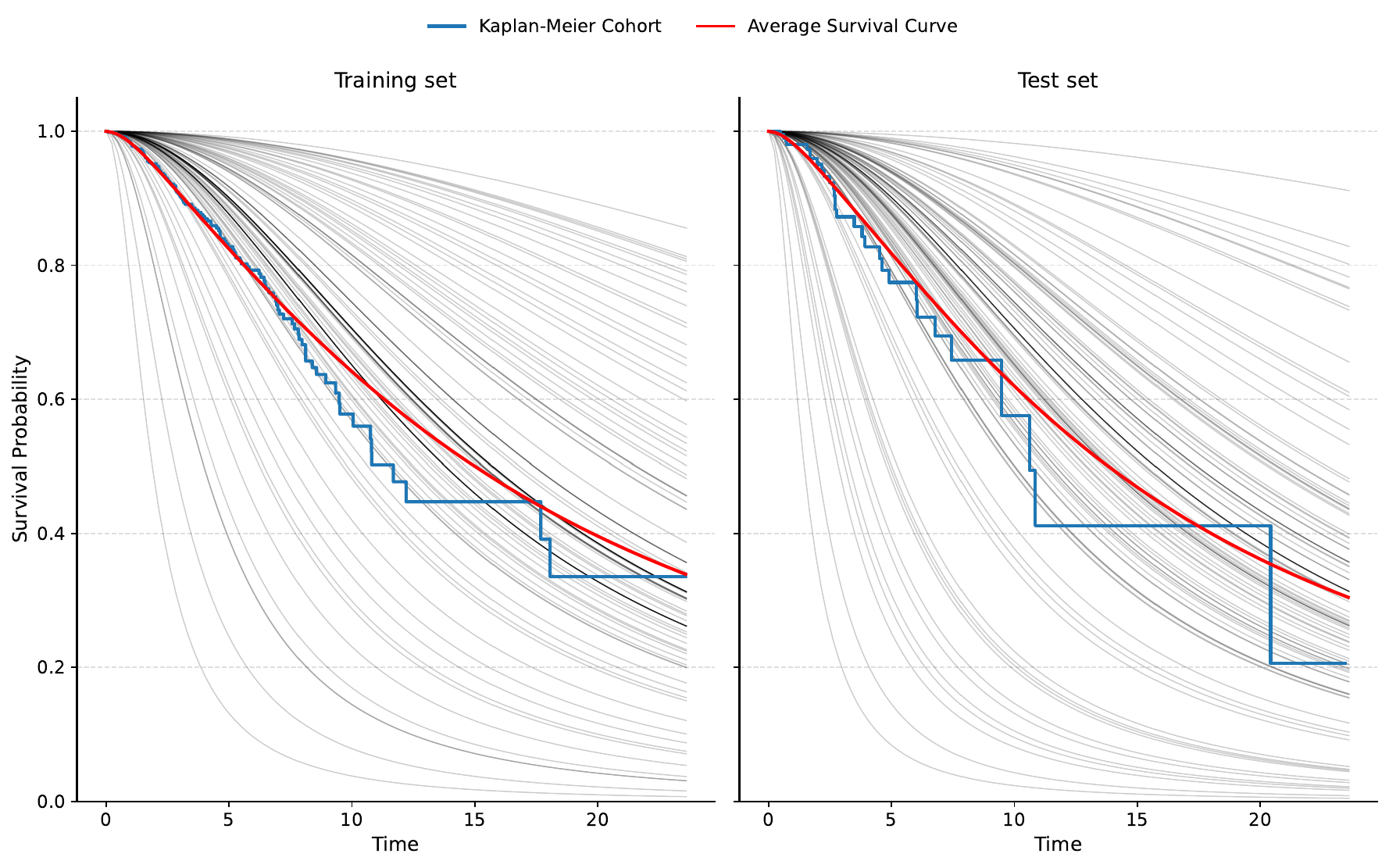}
	\caption{Predicted survival curves, alongside their average curve juxtaposed to the Kaplan-Meier estimates.}
	\label{fig:loglogistic_kaplan_meier}
\end{figure}

%A final way to assess a survival model goodness-of-fit is given by computing its Cox-Snell (CS) residuals, $\text{CS}_i = H(y_i ; \widehat{\boldsymbol \phi}_i) = -\log S(y_i; \widehat{\boldsymbol \phi}_i)$, that represent the predicted cumulative hazard functions evaluated at their respective observed lifetime. Under a correct model specification, these residuals are distributed according to an exponential distribution with rate parameter equal to one. Therefore, a useful plot to verify the model's fit is given by plotting the empirical Nelson Aalen (NA) cumulative hazard estimates for $\text{CS}_i$ against the observed residuals. Under a correct model specification, those points should lie around the identity line. Figure \ref{fig:cox_snell_plot} shows this plot for the Log-logistic final model. As expected, the obtained NA estimator tends to be very close to the identity line. 

A definitive method for evaluating the global fit of a survival model is the analysis of its Cox-Snell (CS) residuals, defined as $R_i^{\text{CS}} = H(y_i ; \widehat{\boldsymbol \phi}_i) = -\log S(y_i; \widehat{\boldsymbol \phi}_i)$, which represent the predicted cumulative hazard evaluated at the observed follow-up time. Under correct model specification, these true residuals behave as a right-censored sample from a standard exponential distribution, $\text{Exp}(1)$. Consequently, model calibration is visually assessed by plotting the Nelson-Aalen (NA) cumulative hazard estimator of the $R_i^{\text{CS}}$ values, strictly retaining their original censoring indicators $\delta_i$, against the residuals themselves. For a perfectly specified model, this NA trajectory should tightly adhere to the identity line. Figure \ref{fig:cox_snell_plot} presents this diagnostic for the optimal Log-logistic model. As anticipated, the estimated NA cumulative hazard is highly congruent with the identity line, visually corroborating the model's structural validity.

\begin{figure}[htbp]
	\centering
	\includegraphics[width=0.9\linewidth]{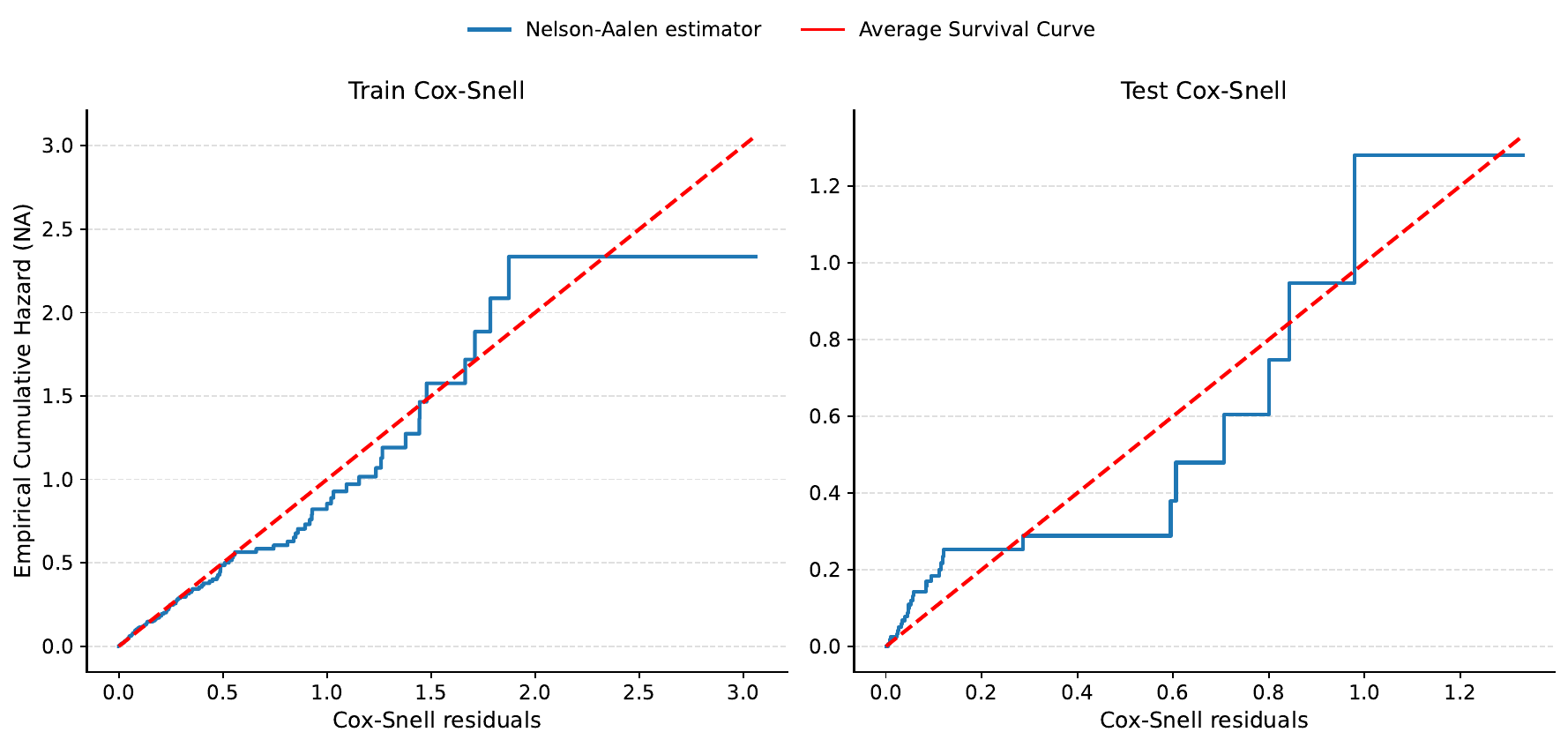}
	\caption{Nelson-Aalen cumulative hazard estimator of the Cox-Snell residuals for the optimal Log-logistic model, evaluated against the unit diagonal to assess global distributional calibration.}
	\label{fig:cox_snell_plot}
\end{figure}

Finally, the randomized normalized quantile residuals for both the training and test sets are visualized via worm plots in Figure \ref{fig:tcga_wormplots}. While the Log-logistic model comfortably outperforms the alternative baseline distributions, this sensitive diagnostic still reveals some deviations from the theoretical standard normal quantiles. Since our analysis was essentially restricted to comparing three base distributions, exploring a wider variety of parametric families remains a practical point for future research.

\begin{figure}[htbp]
	\centering
	\includegraphics[width=0.9\linewidth]{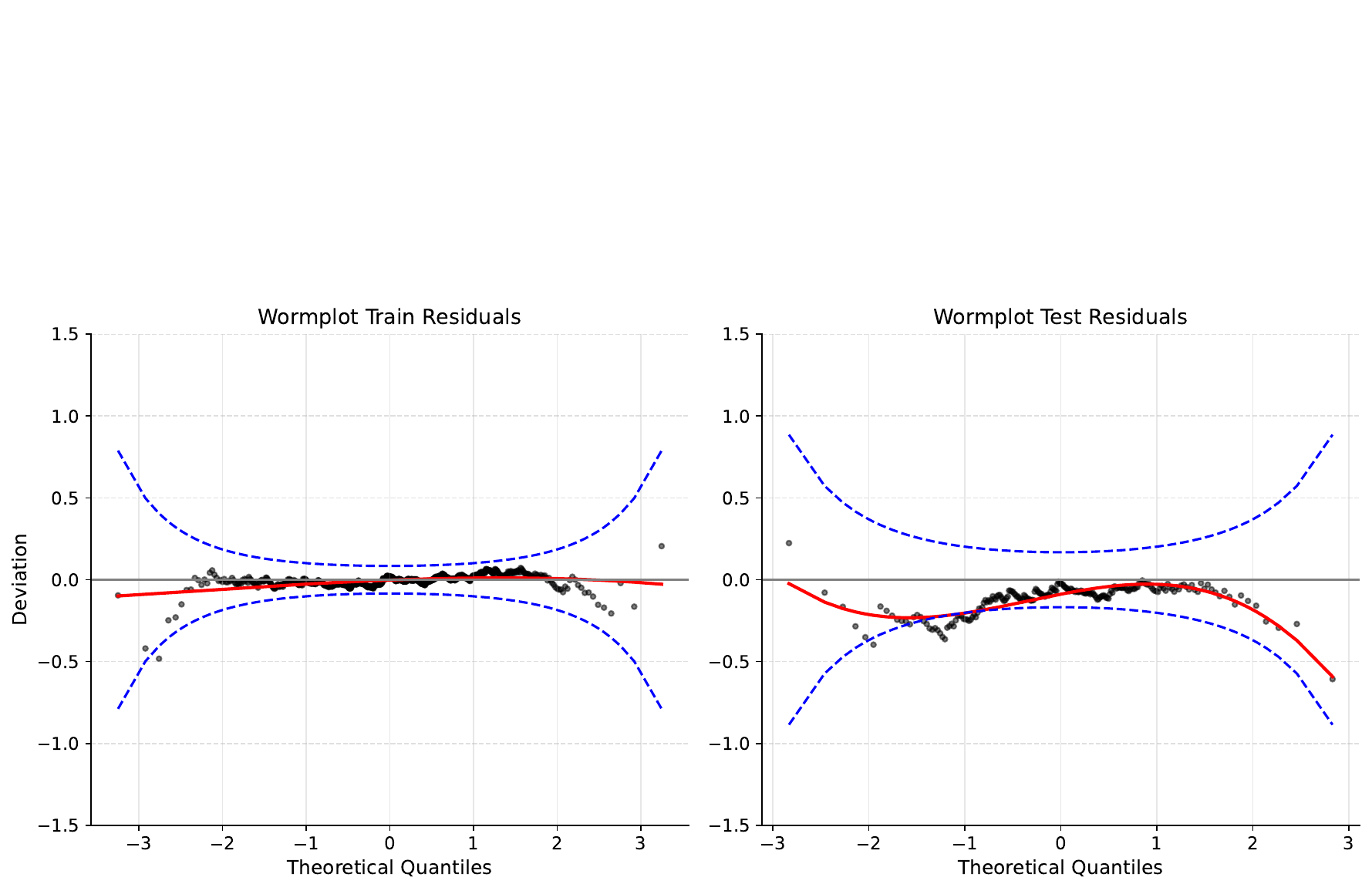}
	\caption{Worm plots of the randomized normalized quantile residuals for the Log-logistic model.}
	\label{fig:tcga_wormplots}
\end{figure}

To practically demonstrate the asymptotic results developed in Section \ref{sec:asymptotic_theory}, we can construct pointwise confidence bands for the estimated survival trajectories. For a fixed grid of evaluation times $t \in (0, y_{\max}]$, we define the scalar-valued transformation function $f_t: \mathbb{R}^p \to [0, 1]$ as:
\begin{equation*}
	f_t(\boldsymbol \phi_i) = S(t \mid \boldsymbol \phi_i),
\end{equation*}
where $S$ is the survival function from the Log-logistic model, evaluated at the maximum likelihood estimate for $\widehat{\boldsymbol \phi}$. Then, using \eqref{eq:asymptotic_distribution_f_phi} is straightforward. Figure \ref{fig:loglogistic_predicted_survival} shows the predicted survival curves for the observed genomic profile of three different patients alongside their respective confidence bands through time.

where $S$ is the survival function of the chosen Log-logistic base distribution. By applying the multivariate Delta method established in \eqref{eq:asymptotic_distribution_f_phi}, we can seamlessly compute the pointwise asymptotic variance, $\widehat{\text{Var}}\left(S(t \mid \widehat{\boldsymbol \phi}_i)\right)$, by evaluating the Jacobian of $f_t$ and the empirical covariance matrix $\widehat{\boldsymbol{\Sigma}}_{\boldsymbol \phi}$ at the maximum likelihood estimate $\widehat{\boldsymbol \phi}_i$.

Figure \ref{fig:loglogistic_predicted_survival} illustrates this capability by displaying the predicted survival trajectories for the observed high-dimensional genomic profiles of three distinct patients. The dashed lines correspond to the 95\% pointwise asymptotic confidence bands derived directly from the model's analytical covariance structure, highlighting the framework's ability to quantify localized, patient-specific predictive uncertainty. The clear visual separation between the individual survival trajectories demonstrates the model's strong discriminatory power, confirming its ability to effectively stratify patient risk based purely on high-dimensional genomic signatures while simultaneously satisfying rigorous parametric distributional assumptions.

\begin{figure}[htbp]
	\centering
	\includegraphics[width=0.9\linewidth]{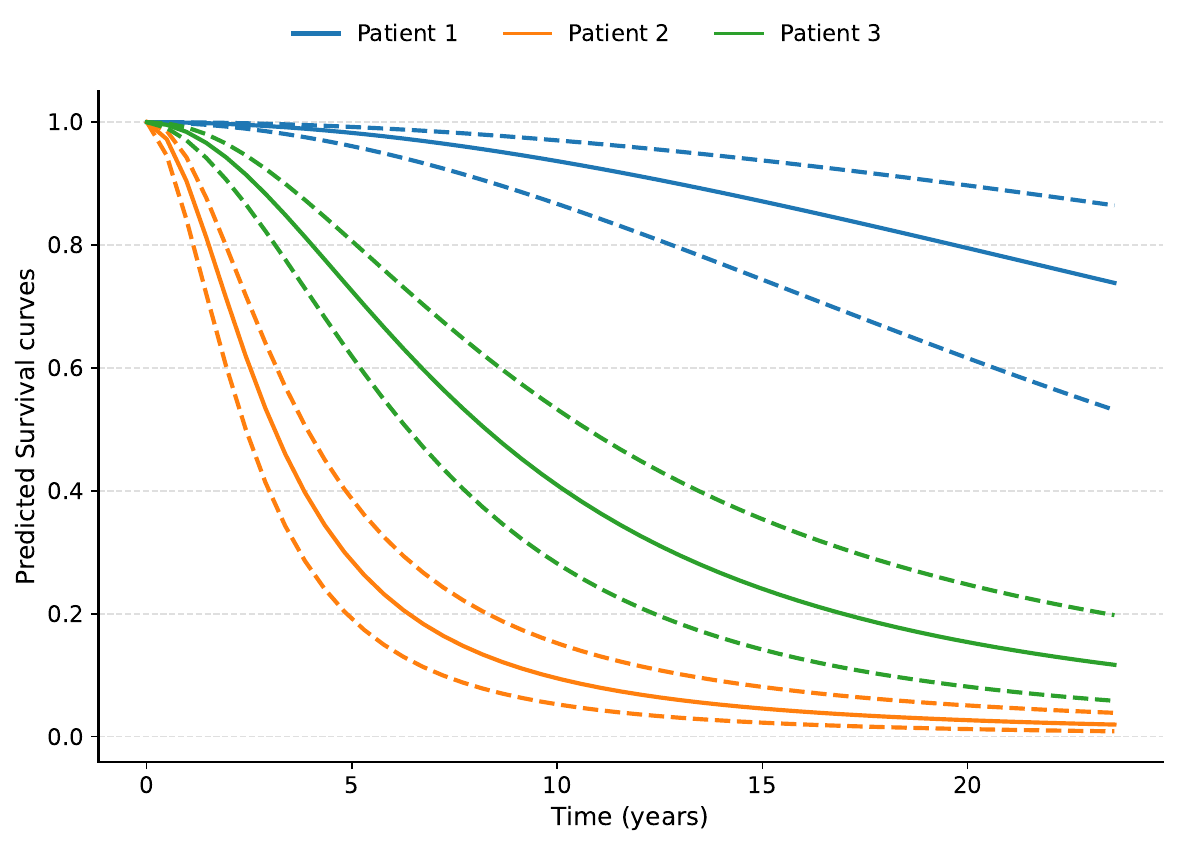}
	\caption{Patient-specific predicted survival curves for three distinctly profiled individuals using the optimal Log-logistic GNDR model.}
	\label{fig:loglogistic_predicted_survival}
\end{figure}

\subsection{Age Prediction from Raw Image Data (UTKFaces)}

For our final application, we demonstrate the GNDR framework's capacity to seamlessly process unstructured, high-dimensional image data by analyzing the UTKFace benchmark dataset \citep{zhang2017utkface}. This dataset comprises 23,708 facial images annotated with the subjects' true ages, representing a highly diverse cohort across age, gender and ethnicity. While UTKFace is predominantly employed in the computer vision literature to benchmark deterministic, point-predictive deep learning architectures, we approach the task strictly through the lens of distributional regression. Specifically, our objective is to map the raw pixel space of a subject's facial image directly to a subject-specific predictive probability distribution for their age.

To capture the inherent heteroscedastic uncertainty in visual age estimation, we assume the conditional distribution of a subject's age, given their image, follows a normal distribution strictly truncated at zero. This parametric choice is theoretically motivated to preserve the highly interpretable location and scale parameters of the Gaussian family while rigorously enforcing the physical boundary of non-negative age. This zero-lower-bound constraint is mathematically critical for the dataset's substantial sub-population of young children, where a standard, unconstrained Gaussian model would inevitably assign positive probability mass to negative ages. Following the GNDR structural formulation established in Equation \eqref{eq:gndr_parameters_structure}, we map the network's unconstrained outputs to the strictly positive pre-truncation location and scale parameters via logarithmic link functions:
\begin{equation*}
	\begin{cases}
		\log(\mu) &= \mathcal{F}_1(\boldsymbol x; \boldsymbol \omega),\\
		\log(\sigma) &= \mathcal{F}_2(\boldsymbol x; \boldsymbol \omega).
	\end{cases}
\end{equation*}

To map the $200 \times 200 \times 3$ RGB input tensor to this bivariate distributional parameter space, we must employ a deep convolutional architecture. Training such highly parameterized networks from random initializations is notoriously unstable and computationally prohibitive for complex spatial topographies. Therefore, we adopt a transfer learning paradigm utilizing the MobileNetV2 architecture \citep{sandler2018mobilenetv2}. Originally pre-trained on large-scale, heterogeneous object recognition tasks (e.g., ImageNet), MobileNetV2 provides robust, hierarchical feature extractors. To specialize these generic representations for the targeted domain of facial feature extraction, we freeze the initial 100 layers of the network. That way, we retain the foundational edge and texture detectors, while specializing all subsequent transferred convolutional blocks.

Table \ref{tab:utkfaces_neuralnet_architecture} details the complete network topology and the dense prediction heads mapped to the final parameter space. During the initial global optimization phase, 1,861,440 of the 2,430,338 total structural weights are actively trained. The vast majority of these optimized weights reside within the unfrozen terminal blocks of MobileNetV2, seamlessly connecting the 1280-dimensional latent spatial feature vector to the final distributional parameters via GELU-activated dense layers. In the subsequent localized fine-tuning phase, we strictly freeze the upstream feature extraction layers and exclusively optimize the final linear layer. This strategic isolation yields a highly tractable sub-space of exactly 130 terminal statistical parameters. By treating the upstream network as a fixed, non-linear basis expansion, we can rigorously extract the analytical asymptotic covariance matrix of these final parameters using the framework established in Equation \eqref{eq:covariance_last_layer}.

Ideally, structural hyperparameters such as the dropout rate would be selected via exhaustive cross-validation, mirroring the procedure utilized for the structured genomic data in Section \ref{sec:tcga_brca_application}. However, the extreme dimensionality of the UTKFace image tensor renders grid-search cross-validation over deep convolutional architectures computationally prohibitive. Consequently, we fix the dropout rate at an empirically robust baseline of $0.3$. This value provides sufficient regularization to the dense prediction heads to mitigate overfitting, without inducing the excessive variance that higher dropout rates can inflict on the final parameter estimation.

\begin{table}[ht]
	\centering
	\caption{Structural summary of the convolutional neural network architecture utilized for the UTKFace application, demonstrating the integration of the MobileNetV2 feature extractor with the distributional prediction heads.}
	{\small
		\begin{tabular}{c|ccc}
			\toprule
			Layer & Output dimension & Activation & Number of weights\\
			\midrule
			MobileNetV2 & 1280 & --- & 2,257,984\\
			Dropout & 1280 & --- & 0\\
			Dense 1 & 128 & GELU & 163,968 \\
			Dense 2 & 64 & GELU & 8,256 \\
			Dense 3 & 2 & Linear & 130 \\
			\midrule
			Total & --- & --- & 2,430,338\\
			\bottomrule
		\end{tabular}
	}
	\label{tab:utkfaces_neuralnet_architecture}
\end{table}

By fitting the GNDR framework in the truncated normal setting and applying the theoretical results from Section \ref{sec:asymptotic_theory}, we obtain a subject-specific $2 \times 2$ asymptotic covariance matrix for $\widehat{\mu}_i$ and $\widehat{\sigma}_i$ for any given image. Establishing this localized uncertainty structure allows us to estimate the asymptotic distribution of any differentiable function of these parameters via the Delta method. Our objective is to construct rigorous predictive age intervals. Let $q_{c}(\mu, \sigma)$ denote the $c$-quantile of the zero-truncated normal distribution. Analytically, this is expressed as:
\begin{equation*}
	q_c(\mu, \sigma) = \mu + \sigma z_{v(c)},
\end{equation*}
where $z_{v(c)} = \Phi^{-1}(v(c))$ represents the $v(c)$-quantile of the standard normal distribution, with the adjusted probability mass $v(c) = \Phi(-\mu / \sigma) + c \left(1 - \Phi(-\mu / \sigma)\right)$. Applying the result in \eqref{eq:asymptotic_distribution_f_phi}, we compute the approximate asymptotic covariance matrix for the vector $[q_{0.025}(\widehat{\mu}, \widehat{\sigma}), q_{0.975}(\widehat{\mu}, \widehat{\sigma})]^\top$. We then construct an uncertainty-adjusted predictive interval by taking the lower $95\%$ confidence bound of $q_{0.025}$ and the upper $95\%$ confidence bound of $q_{0.975}$.

For this application, the UTKFace dataset was partitioned into a training set ($n = 15,173$), a validation set used strictly for early-stopping regularization ($n = 3,793$), and a held-out test set ($n = 4,742$). Evaluating these uncertainty-adjusted predictive intervals, we observe empirical coverage probabilities of $97.21\%$ on the training set. The elevated in-sample coverage perfectly aligns with statistical theory: by incorporating the parameter estimation uncertainty into the interval boundaries, the constructed bounds are intrinsically more conservative than naive plug-in predictive intervals, guaranteeing a coverage strictly exceeding the nominal $95\%$ in-sample. While classical statistical modeling typically halts at in-sample calibration, our rigorous out-of-sample evaluation reveals an expected generalization gap, dropping the coverage to $\approx 89.35\%$. Given the extreme dimensionality of the unstructured pixel space, the inherent ambiguity of visual age estimation and the slight overfit on the training set caused by the second training step, this represents a highly robust predictive capacity.

Figure \ref{fig:predictive_utkfaces_ages} illustrates six diverse test subjects from the UTKFaces dataset alongside their respectively predicted zero-truncated normal densities and estimated uncertainty-adjusted tolerance intervals for the ages. This visual diagnostic demonstrates the GNDR framework's capacity to map complex spatial topologies directly to rigorous inferential uncertainty. Notably, the framework effectively captures the inherent heteroscedasticity of the visual estimation task: lower-quality or structurally ambiguous images (e.g., the bottom-left panel) naturally induce a higher predictive variance, whereas high-resolution, well-illuminated portraits (e.g., the top-right panel) yield substantially narrower intervals. Because we did not manually engineer these quality metrics, this confirms that the fine-tuned convolutional layers successfully internalized image resolution and structural ambiguity, mapping these latent features directly into the distributional scale parameter $\sigma$.

\begin{figure}[htbp]
	\centering
	\includegraphics[width=0.9\linewidth]{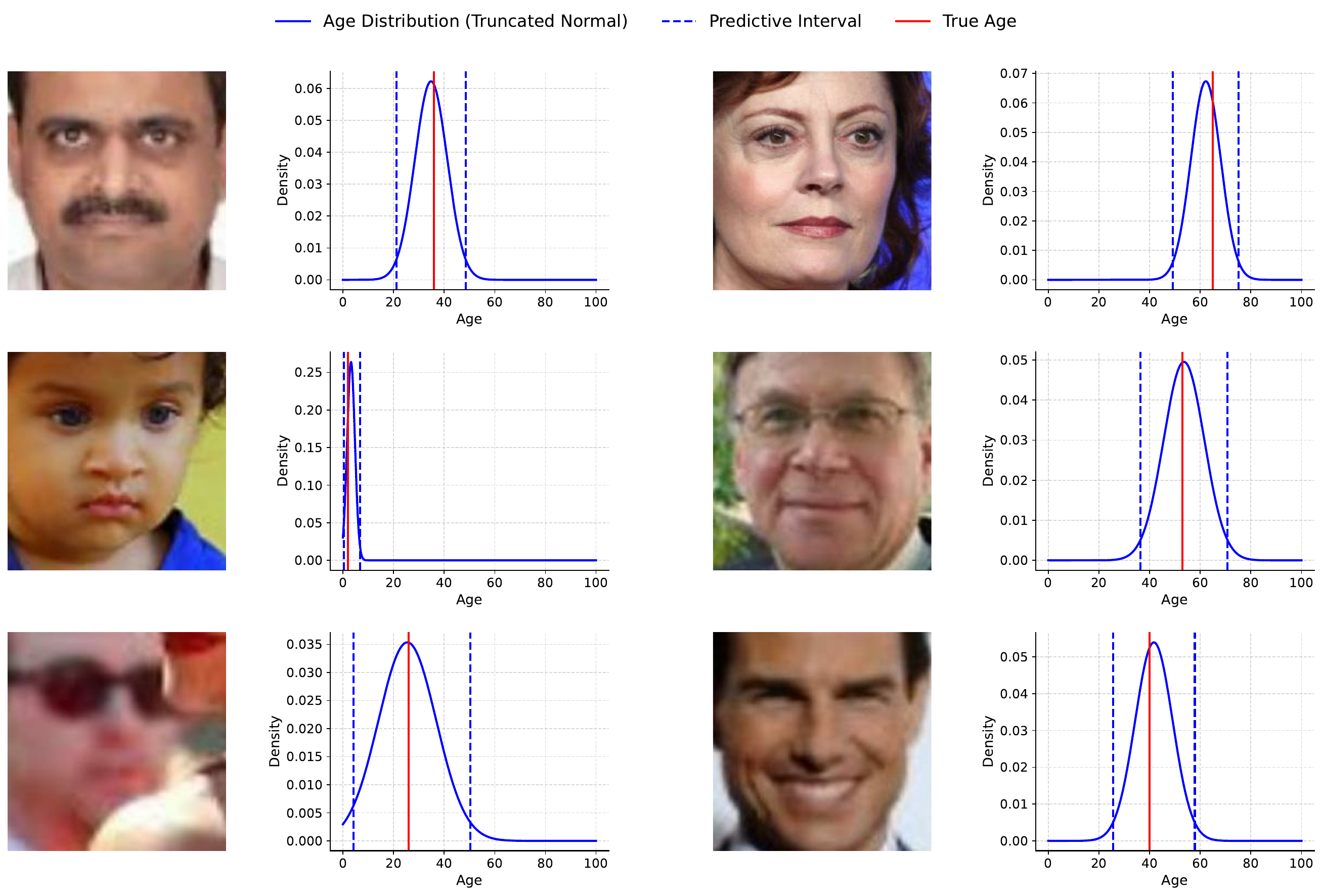}
	\caption{Out-of-sample age prediction for six diverse subjects from the UTKFace test set.}
	\label{fig:predictive_utkfaces_ages}
\end{figure}

\section{Concluding Remarks}
\label{sec:conclusions}

In this article, we introduced the Generalized Neural Distributional Regression (GNDR) framework, a novel statistical paradigm that seamlessly embeds deep neural architectures into the parametric space of classical probability distributions. Historically, the statistical modeling landscape has been bifurcated: classical regression frameworks offer rigorous interpretability and asymptotic guarantees but fail to scale to unstructured, high-dimensional data, while deep learning excels in complex feature extraction but largely operates as a highly parameterized, deterministic black box. The GNDR framework effectively bridges this divide. By utilizing neural networks strictly as non-linear basis expansions for parametric link functions, we preserve the structural elegance of classical distribution theory while unlocking the capacity to process unstructured tensors, such as large-scale genomic arrays and raw pixel imagery.

Crucially, we have demonstrated that this integration does not require abandoning traditional statistical rigor. By establishing strict conditions for structural identifiability and defining a stable, two-step maximum likelihood optimization procedure, we proved that the resulting estimators maintain asymptotic consistency. Furthermore, by evaluating the Fisher Information matrix of the isolated terminal layers, we established a mathematically sound pipeline for uncertainty quantification. Through the multivariate Delta method, the GNDR framework successfully generates localized, subject-specific confidence bands and tolerance intervals, a critical requirement for precision medicine and high-stakes predictive modeling.

Empirically, the framework transcends the structural limitations of foundational methodologies, such as the GAMLSS family. While GAMLSS relies heavily on additive functional forms that struggle with ultra-high dimensionality, GNDR natively ingests complex spatial and non-linear topologies. This was definitively illustrated across our applications: resolving complex interaction surfaces in hospital stay durations, capturing heteroscedastic survival and cure rates from breast cancer genomic profiles, and mapping unstructured facial features directly to zero-truncated normal parameters.

Looking forward, this methodology opens a highly promising frontier for future statistical research. Methodological extensions could explore the integration of bespoke, heavy-tailed baseline distributions, the adaptation of the framework for longitudinal or clustered data via neural mixed-effects models, and the investigation of alternative regularization penalties for the non-linear feature spaces.

To catalyze this future research and ensure the broad accessibility of the GNDR framework, we have developed the open-source Python package \textit{thetaflow}. Implementing generalized neural distributional models from scratch requires navigating highly unstable tensor calculus, specialized loss functions for right-censoring, and dynamic gradient tracking. The \textit{thetaflow} package natively abstracts these computational complexities. It empowers applied researchers and statisticians to easily construct, fit, and evaluate highly flexible distributional models, bridging the gap between strict statistical rigor and modern deep learning infrastructure without the friction of manual tensor manipulation.

 \section*{Funding}

This study was financed, in part, by the São Paulo Research Foundation (FAPESP), Brasil. Process Number \#2025/10988-0.

\begin{appendices}
	\section{Appendix}
	\label{appendix}
	
	\begin{proof}[Proof of Theorem \ref{thm:asymptotic_normality}]
		Conditional on the frozen upstream latent feature representations, $\boldsymbol{v}(\boldsymbol{x})$, the optimization space in the second training stage of the GNDR framework formally reduces to a finite-dimensional Vector Generalized Linear Model (VGLM). Consequently, to establish asymptotic normality, we must verify the standard maximum likelihood regularity conditions outlined by \citet{yee2015vector}. Provided the chosen baseline distribution $\mathcal{D}$ is suitably regular, the GNDR framework satisfies these conditions as follows:
		
		\begin{itemize}
			\item \textbf{Regularity Condition I (Fixed Dimension):} The dimension of $\boldsymbol \theta_\text{last}$ is fixed. The dimension $q = b + M p_\mathcal{F} + p_\mathcal{F}$ depends strictly on the size of the unconstrained structured vector $b$, the number of active neural distributional parameters $p_\mathcal{F}$, and the fixed width of the terminal hidden layer $M$. None of these architectural constants grow with the sample size $n$. 
			
			\item \textbf{Regularity Condition II (Identifiability):} The parameter $\boldsymbol \theta_\text{last}$ must be identifiable. This condition is satisfied through two components: first, the base distribution $\mathcal{D}$ must belong to an identifiable family. Second, the concatenated design matrix formed by the classical covariates $\boldsymbol{z}$ and the learned latent features $\boldsymbol{v}(\boldsymbol{x})$ must be of full column rank. This algebraic requirement is naturally satisfied by our structural restriction (Section 2.2) that classical and neural features mapping to the same parameter must occupy strictly disjoint covariate spaces, preventing perfect collinearity.
			
			\item \textbf{Regularity Condition III (Common Support):} The support of the distribution $\mathcal{D}$ must not depend on the parameter vector $\boldsymbol \theta_\text{last}$. In the GNDR framework, the topological boundaries of the response variable $Y$ (e.g., $y \in (0, \infty)$ for the Weibull and Log-logistic models, or $y \in [0, \infty)$ for the truncated normal) are mathematically fixed and remain completely independent of the dynamic location or scale parameters predicted by $\boldsymbol \theta_\text{last}$.
			
			\item \textbf{Regularity Conditions IV \& V (Parameter Space):} The unconstrained parameter space $\Omega_{\boldsymbol \theta_\text{last}}$ is an open set, and the true parameter vector $\boldsymbol \theta_\text{last}^{(0)}$ lies strictly in its interior. Because the final dense layer utilizes a linear activation function and the classical fixed effects are unbounded, the parameter space is the entire real coordinate space, $\Omega_{\boldsymbol \theta_\text{last}} = \mathbb{R}^q$, which trivially satisfies both conditions.
			
			\item \textbf{Regularity Conditions VI \& VII (Smoothness and Information):} The first three derivatives of the log-likelihood must exist, and the order of integration and differentiation must be interchangeable. By restricting our framework to baseline distributions with smooth, twice-continuously differentiable log-likelihood surfaces (excluding non-smooth families such as the Laplace distribution ) and utilizing strictly differentiable link functions, these analytic conditions hold.
		\end{itemize}
		
		Because the restricted, conditional log-likelihood satisfies these fundamental regularity conditions, the conditional maximum likelihood estimator $\widehat{\boldsymbol \theta}_\text{last}$ is an exact root of the score equation. Therefore, by classical maximum likelihood theory, $\widehat{\boldsymbol \theta}_\text{last}$ is asymptotically consistent and converges in distribution to a multivariate normal:
		$$\sqrt{n} \left(\widehat{\boldsymbol \theta}_\text{last} - \boldsymbol \theta_\text{last}^{(0)} \right) \xlongrightarrow{d} \mathcal{N}_q\left(\mathbf{0}, \mathcal{I}^{-1}(\boldsymbol \theta_\text{last}^{(0)})\right)$$
		concluding the proof.
	\end{proof}
	
\end{appendices}

\newpage

\bibliographystyle{Chicago}
\bibliography{references}

\end{document}